\documentclass{article}

\usepackage{arxiv}

\usepackage[dvipsnames]{xcolor}
\usepackage[pdftex]{hyperref}
\hypersetup{colorlinks,
linkcolor=MidnightBlue,  % colour of links to eqns, tables, sections, etc
urlcolor=MidnightBlue,   % colour of unboxed URLs
citecolor=NavyBlue  % colour of citations linked in text
}

\usepackage{hyperref}       % hyperlinks
\usepackage{url}            % simple URL typesetting
\usepackage{booktabs}       % professional-quality tables
\usepackage{amsfonts}       % blackboard math symbols
\usepackage{nicefrac}       % compact symbols for 1/2, etc.
\usepackage{microtype}      % microtypography
\usepackage{xcolor}         % colors

\usepackage{amsmath,amssymb,amsthm}
\usepackage{bm}
\usepackage{bbm}
\usepackage{comment}         
\usepackage{multicol}
\usepackage{multirow}
\usepackage{caption}
\usepackage{natbib}
\usepackage{braket}
\usepackage{mathtools}
\mathtoolsset{showonlyrefs,showmanualtags} % equation 
\usepackage{algorithm}
\usepackage{algorithmic}
\usepackage{enumitem}

\theoremstyle{plain}

\newtheorem{theorem}{Theorem}[section]
\newtheorem{lemma}[theorem]{Lemma}

\newtheorem{definition}[theorem]{Definition}
\newtheorem{assumption}[theorem]{Assumption}

\newtheorem*{remark}{Remark}

\usepackage{color}

\usepackage{graphicx}

\DeclareMathOperator*{\argmax}{arg\,max}

\DeclareMathOperator*{\supp}{supp}

\usepackage{authblk}
\allowdisplaybreaks

\title{Adaptive Experimental Design for Policy Learning}

\renewcommand{\shorttitle}{Adaptive Experimental Design for Policy Learning}

\author[1]{Masahiro Kato}
\author[2]{Kyohei Okumura}
\author[3]{Takuya Ishihara}
\author[4]{Toru Kitagawa}

\affil[1]{Data Analytics Team, Mizuho–DL Financial Technology, Co., Ltd.}
\affil[2]{Department of Economics, Northwestern University}
\affil[3]{Graduate School of Economics and Management, Tohoku University}
\affil[4]{Department of Economics, Brown University}

\begin{document}

\maketitle

\begin{abstract}
This study investigates the contextual best arm identification (BAI) problem, aiming to design an adaptive experiment to identify the best treatment arm conditioned on contextual information (covariates). We consider a decision-maker who assigns treatment arms to experimental units during an experiment and recommends the estimated best treatment arm based on the contexts at the end of the experiment. The decision-maker uses a policy for recommendations, which is a function that provides the estimated best treatment arm given the contexts. In our evaluation, we focus on the worst-case \emph{expected regret}, a relative measure between the expected outcomes of an optimal policy and our proposed policy. We derive a lower bound for the expected simple regret and then propose a strategy called \emph{Adaptive Sampling-Policy Learning} (PLAS). We prove that this strategy is minimax rate-optimal in the sense that its leading factor in the regret upper bound matches the lower bound as the number of experimental units increases.
\end{abstract}

%Evidence-based targeting has been a topic of growing interest among the practitioners of policy and business. Formulating decision-maker's policy learning as a fixed-budget best arm identification (BAI) problem with contextual information, we study an optimal adaptive experimental design for policy learning with multiple treatment arms. In the sampling stage, the planner assigns treatment arms adaptively over sequentially arriving experimental units upon observing their contextual information (covariates). After the experiment, the planner recommends an individualized assignment rule to the population. Setting the worst-case \emph{expected regret} as the performance criterion of adaptive sampling and recommended policies, we derive its asymptotic lower bounds, and propose a strategy, \emph{Adaptive Sampling-Policy Learning strategy} (PLAS), whose leading factor of the regret upper bound aligns with the lower bound as the size of experimental units increases.

\section{Introduction}

In this study, we design an \emph{adaptive experiment for policy learning}. We consider the problem of decision-making given multiple \emph{treatment arms}, such as arms in slot machines, diverse therapies, and distinct unemployment assistance programs. The primary objective is to identify the \emph{best treatment arm} for individuals given covariates, often referred to as context, at the end of an experiment. For this purpose, we aim to learn a policy that recommends the conditional best treatment arm by using data adaptively collected via an experiment.

In our setting, at each round of an adaptive experiment, a decision-maker sequentially observes a context (covariate) and assigns one of the treatment arms to a research subject based on past observations and the observed contexts. At the end of the experiment, the decision-maker recommends an estimated best treatment arm conditional on a context.

We design the adaptive experiment by developing a strategy that the decision-maker follows. A strategy is defined as a pair of a sampling rule and a recommendation rule. In the adaptive experiment, the decision-maker assigns treatment arms following the sampling rule during the experiment and recommends a treatment arm following the recommendation rule at the end of the experiment.

We measure the performance of a strategy using the expected simple regret, which is the difference between the maximum expected outcome that could be achieved with full knowledge of the distributions of the treatment arms and the expected outcome of the treatment arm recommended by the decision-maker's strategy. Our goal is to develop a strategy that minimizes the expected simple regret.

The challenge in our problem arises from the need for context-specific recommendations. Unlike prior studies that do not consider contextual information, developing a model that captures the relationship between context and outcomes becomes imperative.

To address this issue, we define a function suggesting a treatment arm given a context as a \emph{policy}. By using a policy, we restrict strategies to ones with recommendation rules using a policy learned from observations obtained from an adaptive experiment. Policy learning has been extensively studied in causal inference and reinforcement learning \citep{dudik2011doubly,Swaminathan2015,Kitagawa2018,AtheySusan2017EPL,ZhouZhengyuan2018OMPL}, but to the best of our knowledge, adaptive experimental design for policy learning has not been fully explored. Note that there are existing studies that address contextual BAI aiming to find the best treatment arm marginalized over the contextual distribution \citep{Russac2021,Kato2021Role,simchi2024experimentation}, motivated by the studies of efficient average treatment effect estimation \citep{Laan2008TheCA,Hahn2011,Kato2020adaptive,cook2023semiparametric}.

Our problem corresponds to a generalization of \emph{best arm identification} \citep[BAI,][]{Bubeck2009,Bubeck2011,Audibert2010}, an instance of the stochastic multi-armed bandit (MAB) problem \citep{Thompson1933,Lai1985}. Therefore, we refer to the problem as contextual fixed-budget BAI, as well as an adaptive experimental design for policy learning.

\paragraph{Contribution.}
We propose a strategy that assigns treatment arms following \emph{Adaptive Sampling} (AS) and recommends a treatment arm using \emph{Policy Learning} (PL). In the AS rule, a decision-maker assigns a treatment arm to an experimental unit based on a probability depending on the variances of the experimental units' outcomes. Because the variances are unknown, the decision-maker estimates them during the experiment and continues updating the assignment probability. At the end of the experiment, the decision-maker trains a policy using observations obtained in the experiment and recommends a treatment arm using the trained policy. We refer to our strategy as the \emph{PLAS strategy}.

To design an optimal strategy, we first develop a lower bound (theoretical limit) for the expected simple regret. Subsequently, we design a strategy and evaluate its upper bound (performance) by comparing it to the lower bound.

In the evaluation, given the inherent uncertainties, we use the minimax criterion for performance assessment, which evaluates the worst-case scenario among a set of distributions. The minimax approach has garnered attention in studies about experimental design, including BAI \citep{Bubeck2009,Bubeck2011,Carpentier2016,Ariu2021,yang2022minimax,Komiyama2022}. A critical quantity in the minimax evaluation is the gap between the expected outcomes of the best treatment and the other suboptimal treatment arms, referred to as the average treatment effects in the literature of causal inference. The worst-case distributions are characterized by gaps approaching zero at a rate of order $\sqrt{T}$ \citep{Bubeck2009,Bubeck2011}, where $T$ denotes the sample size (total rounds of an adaptive experiment).

Our research identifies the leading factor in the lower bound as the variances of potential outcomes, also providing a variance-dependent sampling rule. We subsequently show that the PLAS strategy is asymptotically minimax optimal, as its foremost factor of the worst-case expected simple regret aligns with the lower bound.

In summary, our contributions include: (i) a lower bound for the worst-case expected simple regret; (ii) the PLAS strategy with a closed-form target assignment ratio, characterized by the variances of outcomes; and (iii) the asymptotic minimax optimality of the PLAS strategy. These findings contribute to a variety of subjects, including decision theory and causal inference, in addition to BAI.

\paragraph{Organization.}
The structure of this paper is as follows: Section~\ref{sec:problem_setting} defines our problem. Section~\ref{sec:minimax_lower} develops lower bounds for the worst-case expected simple regret. Section~\ref{sec:track_aipw} introduces the PLAS strategy, and Section~\ref{sec:asymp_opt} presents upper bounds for the proposed strategy and its asymptotic minimax optimality. Further related work is introduced in Appendix~\ref{sec:related_work}.

\section{Problem Setting}
\label{sec:problem_setting}
This study considers an adaptive experiment with a fixed budget (sample size) $T\in\mathbb{N}$, a set of treatment arms $[K] \coloneqq \{1, 2, \dots, K\}$, and a decision-maker who aims to identify the context-conditional best treatment arm. In each round $t \in [T]$, experimental units sequentially visit, and the decision-maker can assign treatment arms to them. At the end of the experiment, the decision-maker recommends an estimated context-conditional best treatment arm. 

\subsection{Potential Outcomes} 
Following the Neyman-Rubin causal model \citep{Neyman1923,Rubin1974}, let $Y^a\in\mathcal{Y}$ be a potential random outcome of treatment arm $a\in [K]$, where $\mathcal{Y}\subset \mathbb{R}$ denotes a set of possible outcomes. We also define $X \in \mathcal{X}$ as a context, also called covariates, that characterizes an experimental unit, where $\mathcal{X}\subset \mathbb{R}^d$ denotes a set of possible $d$-dimensional covariates. We define a tuple $(Y^1, \dots, Y^K, X)$ as a $(\mathcal{Y}^K \times \mathcal{X})$-valued random variable following a probability distribution $P \in \mathcal{M}(\mathcal{Y}^K \times \mathcal{X})$, where $\mathcal{M}(\mathcal{T})$ denotes the set of all Borel probability measures on a topological space $\mathcal{T}$.
Let $\mathbb{E}_P$ and $\mathrm{Var}_{P}$ be the expectation and variance operators under $P$, respectively.
For each $P \in\mathcal{P}$, let us denote the first and second moments of the potential outcome as $\mu^a(P)(x) \coloneqq \mathbb{E}_P[Y^a \mid X=x]$ and $\nu^a(P)(x) \coloneqq \mathbb{E}_{P}[(Y^a)^2\mid X=x]$, respectively. 

In our setting, a distribution $P$ of $(Y^1,\dots, Y^K, X)$ belongs to a bandit model $\mathcal{P}$ defined below.
We refer to the bandit model as a location-shift bandit model.
\begin{definition}[Location-shift bandit model]
\label{def:ls_bc}
Let $(\sigma^a)^2 \colon \mathcal{X} \to (0, +\infty)$ be a function that is exogenously given but \emph{unknown} to the decision-maker. Let $\zeta$ be a distribution of $X$ that is exogenously given but \emph{unknown} to the decision-maker, which is supported on $\mathcal{X}$. 
Then, a \emph{location-shift bandit model} $\mathcal{P}$ is defined as follows:
\begin{align}
    &\mathcal{P} \coloneqq \mathcal{P}_\zeta \coloneqq \Big\{P \in \mathcal{M}(\mathcal{Y}^K \times \mathcal{X})
    \colon
    \forall a \in [K]\ \ \ \forall x \in \mathcal{X},\ \ \ \mathrm{Var}_P(Y^a \mid X=x) = \sigma^a(x), \\
    &\ Y^a - \mu^a(P)(x)\ \mathrm{is}\ \mathrm{zeo}\mathchar`-\mathrm{mean}\ \mathrm{subgaussian}\ \mathrm{given}\ x,\ \ \ \mu^a(P)(x) \in (-\infty, +\infty),\ \ \ \mathrm{marg}_{\mathcal{X}}(P) = \zeta
    \Big\},
\end{align}
where $\mathrm{marg}_{\mathcal{X}}(P)$ denotes the marginal distribution of $P$ on $\mathcal{X}$.
\end{definition}
Our lower and upper bounds are characterized by $\sigma^a(x)$, given independently of the experiment. Location-shift models are a commonly employed assumption in statistical analysis \citep{LehmCase98}. A key example is a Gaussian distribution, where for all $P$, the variances are fixed and only mean parameters shift. Note that we can omit the condition $|\nu^a(P)(x)| < \overline{C}$ from the boundedness of the variance. However, we introduce it to simplify the definition of our strategies.

\subsection{Adaptive Experiment}
We consider a decision-maker who aims to identify the best treatment arm 
\[a^*(P)(x) \in \argmax_{a\in [K]}\mu^a(P)(x)\] 
for each context $x \in \mathcal{X}$ through an adaptive experiment.
A fixed number of rounds $T$, called the budget or sample size, is exogenously given.
At each round $t \in [T],$ the decision-maker uses the following procedure:
\begin{enumerate}[topsep=0pt, itemsep=0pt, partopsep=0pt, leftmargin=*]
    \item A potential outcome $(Y^1_t, Y^2_t, \dots, Y^K_t, X_t)$ is generated from $P$.
    \item The decision-maker observes a context $X_t$.
    \item The decision-maker assigns treatment arm $A_t$ to an experimental unit based on past observations $(X_s, A_s, Y_s)^{t-1}_{s=1}$ and the context $X_t$. 
    \item The decision-maker observes an outcome of the assigned treatment arm, $Y_t = \sum_{a\in[K]}\mathbbm{1}[A_t = a]Y^a_t$.
\end{enumerate}
This process is referred to as the \emph{exploration phase}.
Note that outcomes from unassigned treatment arms remain unobserved.
This setting is called the bandit feedback or Rubin causal model \citep{Neyman1923,Rubin1974}. 
Through the exploration phase, the decision-maker obtains observations $(X_t, A_t, Y_t)^T_{t=1}$. After round $T$, the decision-maker recommends an estimated best treatment arm $\widehat{a}_T(x)$ for each context $x$ given the observations $(X_t, A_t, Y_t)^T_{t=1}$.

\subsection{Strategy with Policy Learning} 
A strategy of the decision-maker defines which treatment arm to assign in each period during the exploration phase and which treatment arm to recommend as an estimated best arm for each context. A strategy is defined as a pair $((A_t)_{t\in[T]}, \widehat{a}_T)$ in which:
\begin{itemize}[topsep=0pt, itemsep=0pt, partopsep=0pt, leftmargin=*]
    \item the \emph{sampling rule} $(A_t)_{t\in[T]}$ determines which treatment arm $A_t$ to assign in round $t$ based on the past observations. In other words, $A_t$ is $\mathcal{G}_t$-measurable, where $\mathcal{G}_t \coloneqq \sigma(X_1,A_1,Y_1, \dots, X_{t-1}, A_{t-1}, Y_{t-1}, X_t)$ for each $t \in [T]$.
    \item the \emph{recommendation rule} $\widehat{a}_T:\mathcal{X}\to [K]$ returns an estimated best treatment arm for each context $\widehat a_T$ based on the observations collected during the exploration phase. In other words, for each $x \in \mathcal{X}$, $\widehat a_T(x)$ is $\mathcal{F}_T$-measurable, where $\mathcal{F}_T \coloneqq \sigma(X_1, A_1, Y_1, \dots, X_T, A_T, Y_T)$.
\end{itemize}

\paragraph{Policy.}
In this study, we impose a restriction on a class of recommendation rules that the decision-maker can use.
We assume that there is an \emph{exogenously} given class of policies $\Pi$, whose typical element is a policy $\pi \colon [K] \times \mathcal{X} \to [0,1]$ that is measurable and satisfies $\sum_a \pi(a, x) = 1$ for each $x$.
Here,
$\pi(a,x)$ denotes a probability that the decision-maker recommends treatment arm $a \in [K]$ as an estimated best treatment arm for context $x$. With a slight abuse of notation, we write $\pi(a \mid x)$ instead of $\pi(a,x)$.
We require the decision-maker to obtain an estimator $\widehat{a}_T(x)$ of $a^*(P)(x)$ as follows:
first, the decision-maker constructs a policy $\pi \in \Pi$ based on the observations collected during the exploration phase; then, $\widehat{a}_T(x)$ is drawn from $\pi(\cdot \mid x)$ for each $x$.

\paragraph{Optimal policy.}
We evaluate the performance of policies via a simple regret.
The \emph{value} of policy $\pi\in\Pi$ under $P$ is the expected outcome when the decision-maker uses a policy $\pi$, which is defined as
\[
Q(P)(\pi) \coloneqq \mathbb{E}_P\left[\sum_{a\in[K]}\pi(a\mid X)\mu^{a}(P)(X); \pi\right],
\] 
and the optimal policy within class $\Pi$ is defined as $
\pi^*(P) \coloneqq \argmax_{\pi\in\Pi} Q(P)(\pi)$. 

Given a strategy with a policy $\widehat{\pi}\in\Pi$ of the decision-maker, we define a simple regret for each context $x \in \mathcal{X}$ under $P \in \mathcal{P}$ as
\begin{align*}
    r_T(P)\left(\widehat{\pi}\right)(x) \coloneqq & \sum_{a\in[K]}\pi^*(P)(a\mid x)\mu^{a}(P)(x) - \mu^{\widehat{a}_T(x)}(P)(x),
\end{align*}
and the marginalized simple regret is defined as 
\begin{align*}
R_T(P)\left(\widehat{\pi}\right) \coloneqq \mathop{\mathbb{E}}_{X \sim \zeta}\left[ r_T(P)\left(\widehat{\pi}\right)(X)\right] 
&= Q(P)(\pi^*(P)) - Q(P)\left(\widehat{\pi}\right). 
\end{align*}
Then, we define the \emph{expected simple regret} as  $\mathbb{E}_P\left[R_T(P)\left(\widehat{\pi}\right)\right]$, 
where the expectation is taken over $\widehat{\pi}$. This expected simple regret is our performance measure of interest. We also refer to the expected simple regret $\mathbb{E}_P\left[R_T(P)\left(\widehat{\pi}\right)\right]$ as the \emph{policy regret}.
The decision-maker aims to identify the best treatment arm with a smaller expected simple regret.

\paragraph{Notation.} Let $o(g(x))$ be Landau's notation, and $f(x) = o(g(x))$ implies that $\forall \varepsilon > 0\ \exists x_0\ \forall x > x_0\colon |f(x)| < \varepsilon g(x)$ holds. Let let $\mathrm{thre}(A, a, b) \coloneqq \max\{ \min\{A,a\},b\}$ be a truncation function.

\section{Regret Lower Bound}
\label{sec:minimax_lower}
This section presents a lower bound on the expected simple regret $\mathbb{E}_P\left[R_T(P)\left(\widehat{\pi}\right)\right]$. The lower bound is provided under weak conditions on the policy. Not only does the lower bound offer insights into the difficulty of the problem, but it also helps argue which sample allocations are optimal.

\subsection{Restriction and Complexity of a Policy Class}
To establish lower bounds, we introduce a moderate precondition related to the strategy space of the decision-maker. Specifically, we require that, in the limit, strategies choose all the arms with an equal probability when, for a given covariate $x$, the expected outcomes associated with all arms are identical. Strategies adhering to this criterion are termed \emph{null consistent strategies}.

\begin{definition}[Null consistent strategy]
We say a strategy is \emph{null consistent} if the following condition is satisfied:
If $\mu^1(P)(x) = \mu^2(P)(x) = \cdots = \mu^K(P)(x)$, then for any $a, b \in [K]$, we have
\[
\Big|\mathbb{P}_{P}\left(\widehat{a}_T(X) = a \mid  X = x\right) - \mathbb{P}_{P}\left(\widehat{a}_T(X) = b \mid  X = x\right)\Big| \to 0 \quad (T\to\infty). 
\]
\end{definition}
Under any null consistent strategies, $\big|\mathbb{P}_{P}\left(\widehat{a}_T(X) = a \mid  X = x\right) - 1/K\big| = o(1)$ holds for each $a\in[K]$ as $T\to\infty$ if $\mu^1(P)(x) = \mu^2(P)(x) = \cdots = \mu^K(P)(x)$. 

Next, we introduce the Natarajan dimension, a metric that measures the complexity of a policy class $\Pi$ \citep{Natarajan1989}. Our lower bounds are characterized by the Natarajan dimension.
\begin{definition}[Natarajan dimension]
We say that $\Pi$ shatters $M$ points $\{s_1, s_2, \dots, s_M\} \subseteq \mathcal{X}$
if there exist $f_1, f_{-1}: \{s_1, s_2, \dots, s_M\}\to [K]$ such that 
\begin{enumerate}[topsep=0pt, itemsep=0pt, partopsep=0pt, leftmargin=*]
    \item for any $j \in [M]$, $f_{-1}(s_j) \neq f_1(s_j)$ holds;
    \item for any $\bm{\sigma} \coloneqq  \{\sigma_1, \sigma_2, \dots, \sigma_M\}\in \{\pm 1\}^M$, there exists a policy $\pi\in\Pi$ such that for any $j\in[M]$, it holds that $\pi(s_j) = \begin{cases}
            f_1(s_j) & \mathrm{if}\ \ \ \sigma_j = 1\\
            f_{-1}(s_j) & \mathrm{if}\ \ \ \sigma_j = -1
        \end{cases}$. 
\end{enumerate}
The Natarajan dimension of $\Pi$, denoted by $d_{\mathrm{N}}(\Pi)$, is the maximum cardinality of a set shattered by $\Pi$. 
\end{definition}
Let $d_{\mathrm{VC}}(\Pi)$ be the \emph{Vapnik-Chervonenkis (VC) dimension} of $\Pi$. 
Note that when $K = 2$, the Natarajan dimension is equivalent to the VC dimension; that is, when $K = 2$, $d_{\mathrm{VC}}(\Pi) = d_{\mathrm{N}}(\Pi)$ holds.  

\subsection{Regret Lower Bounds}
We derive the following lower bounds of the expected simple regret, which hold for any null consistent strategies and depend on the complexity of a policy class $\Pi$, measured by the Natarajan dimension. The proof is shown in Appendix~\ref{sec:proof_thms}. 
\begin{theorem}
\label{thm:null_lower_bound}
There exists a distribution $\zeta$ on $\mathcal{X}$ such that for any $K \geq 2$, any null consistent strategy $\pi$ with a policy class $\Pi$ such that $d_{\mathrm{N}}(\Pi) = M$ satisfies
\begin{align*}
&\sup_{P\in\mathcal{P}_\zeta} \sqrt{T}\mathbb{E}_P\left[R(P)\left(\pi\right)\right] \geq \frac{1}{8}\mathop{\mathbb{E}}_{X \sim \zeta}
\left[\sqrt{M\sum_{a\in[K]}\left(\sigma^a(X)\right)^2}\right] + o(1)\quad \mathrm{as}\ \ T\to \infty,
\end{align*}
where $\mathop{\mathbb{E}}_{X \sim \zeta}$ denotes the expectation of a random variable $X$ under the probability distribution $\zeta$. 
\end{theorem}
Our lower bounds also depend on the variances of outcomes $Y^a$. 

When $K = 2$, we can obtain a tighter lower bound than the one in Theorem~\ref{thm:null_lower_bound}.
The proof is provided in Section~\ref{sec:proof_ref}.
\begin{theorem}
\label{thm:null_lower_bound_ref}
There exists a distribution $\zeta$ on $\mathcal{X}$ such that for $K = 2$, any null consistent strategy $\pi$ with policy class $\Pi$ such that $d_{\mathrm{VC}}(\Pi) = M$ satisfies
\begin{align*}
    &\sup_{P\in\mathcal{P}_\zeta} \sqrt{T}\mathbb{E}_P\left[R(P)\left(\pi\right)\right] \geq 
    \frac{1}{8}\mathop{\mathbb{E}}_{X \sim \zeta}\left[\sqrt{M\left(\sigma^{1}(X) + \sigma^{2}(X)\right)^2}\right] + o(1)\quad \mathrm{as}\ \ T\to \infty.
\end{align*}
\end{theorem}

Here, note that $\sum_{a\in\{1,2\}}\left(\sigma^a(x)\right)^2 \leq \left(\sigma^{1}(x) + \sigma^{2}(x)\right)^2$
holds for each $x\in\mathcal{X}$. 
Therefore, when $K = 2$, we use the lower bound in Theorem~\ref{thm:null_lower_bound_ref} and when $K\geq 3$, we use the one in Theorem~\ref{thm:null_lower_bound}.

\section{The PLAS Strategy}
\label{sec:track_aipw}
Our strategy consists of the following sampling and recommendation rules. First, we define a \emph{target assignment ratio}, which is an ideal treatment assignment probability. At each round, $t = 1,2,\dots, T$, our sampling rule randomly assigns a unit to a treatment arm with a probability identical to an estimated target assignment ratio. After the final round $T$, our recommendation rule recommends a treatment arm with the highest value of a policy trained by maximizing empirical policy value. We refer to our strategy as the PLAS strategy. Our strategy depends on hyperparameters $\overline{C} \in (0, \infty)$, which are introduced for technical purpose in the proof and can be set as sufficiently large values. We show a pseudo-code in Algorithm~\ref{alg}. 

\subsection{Optimal Target Assignment Ratio}
\label{sec:target_assignment}
We first define a target assignment ratio. The target assignment ratio is the expected value of the sample average of $A_t$ of a strategy ($\frac{1}{T}\sum^T_{t=1}\mathbb{E}_P\left[\mathbbm{1}[A_t = a]\mid X_t = x\right]$) under which a leading factor of its expected simple regret aligns with that of our derived lower bound.

\begin{definition}[Target assignment ratio]
\label{def:est_target_assignment}
    When  $K=2$, for each $a\in[K] = \{1, 2\}$, we define the target assignment ratio $w^*$ as
\begin{align}
\label{eq:targetratio1}
        w^*(a\mid x) = \frac{\sigma^{a}(x)}{\sigma^{1}(x) + \sigma^{2}(x)}.
\end{align}
When $K \geq 3$, for each $a\in[K]$, we define the target assignment ratio as
\begin{align}
\label{eq:targetratio2}
    w^*(a\mid x) = \frac{ \left(\sigma^{a}(x)\right)^2}{\sum_{b\in[K]} \left(\sigma^{b}(x)\right)^2}. 
\end{align}
\end{definition}

This target assignment ratio is given in the course of proving Theorem~\ref{thm:null_lower_bound}, in which we solve $\min_{\bm{w}\in \mathcal{W}}
\left\{
    \sum_{s\in\mathcal{S}}
    \max_{a \in [K]}
    \left\{\sqrt{\frac{\left(\sigma^{a}(s)\right)^2}{w(a\mid s)}}\right\}
    \right\}$, 
where $\mathcal{S} \coloneqq  \{s_1, s_2, \dots, s_M\} \subseteq \mathcal{X}$, and $\mathcal{W}$ is the set of all measurable functions $w:\mathcal{X}\times [K] \to (0, 1)$ such that $\sum_{a\in[K]} w(a\mid x) = 1$ for each $x\in \mathcal{X}$. The solutions $w^*$, whose explicit forms appear in \eqref{eq:targetratio1} and \eqref{eq:targetratio2}, work as a conjecture for $\frac{1}{T}\sum^T_{t=1}\mathbb{E}_P\left[\mathbbm{1}[A_t = a]\mid  X_t = x\right]$.

These target assignment ratios are ex-ante unknown to the decision-maker since the variance $(\sigma^a(x))^2$ is unknown.
Therefore, the decision-maker needs to estimate it during an adaptive experiment and employ the estimator as a probability of assigning a treatment arm. 

\begin{remark}[Efficiency gain]
For each $a\in[K]$, let $\left(\sigma^a\right)^2$ be the unconditional variance of $Y^a_t$, and $w:[K]\to[0, 1]$ be an assignment ratio such that $\sum_{a\in[K]}w(a) = 1$ when we cannot utilize the contextual information. Then, the following inequality holds:
\begin{align*}
    &\sup_{P\in\mathcal{P}} \sqrt{T}\mathbb{E}_P\left[R(P)\left(\widehat{\pi}\right)\right] \geq \frac{1}{8}\inf_{w\in\mathcal{W}}\max_{a\in[K]}\sqrt{M{\left(\sigma^{a}\right)^2}/{w(a)}}\\
    &\geq \frac{1}{8}\inf_{w\in\mathcal{W}}\max_{a\in[K]}\sqrt{M\mathop{\mathbb{E}}_{X \sim \zeta}\left[{\left(\sigma^{a}(X)\right)^2}/{w(a\mid X)}\right]} \geq \frac{1}{8}\inf_{w\in\mathcal{W}}\max_{a\in[K]}\mathop{\mathbb{E}}_{X \sim \zeta}\left[\sqrt{M{\left(\sigma^{a}(X)\right)^2}/{w(a\mid X)}}\right].
\end{align*}
This result implies that we can minimize a lower bound by using contextual information; that is, strategies utilizing contextual information are more efficient than ones not utilizing contextual information. 
\end{remark}

\subsection{Sampling Rule with Adaptive Sampling (AS)}
In this section, we describe our sampling rule, referred to as athe AS rule. 

In round $t \leq K$, the strategy chooses $A_t = t$, i.e., each arm is pulled once as initialization.
In round $t > K$, given an estimated target assignment ratio $\widehat{w}_{t}(a\mid x)$, we assign treatment arm $a$ with probability $\widehat{w}_{t}(a\mid X_t)$. Below, we describe the construction of $\widehat{w}_{t}(a\mid x)$.

In each round $t$, we estimate $w^*$ using the past observations. We first construct estimators $\widehat{\mu}^a_{t}(x)$ and $\widehat{\nu}^a_{t}(x)$ of the first moment $\mu^a(P)(x)$ and the second moment $\nu^a(P)(x)$ of $Y^a$. The estimators constructed to converges to the true functions with probability one, as stated in Assumption~\ref{asm:consistent}, and their absolute values are bounded by $\overline{C}$. Then, given these estimators, we estimate the variances as $\left(\widehat{\sigma}^a_{t}(x)\right)^2 = \mathrm{thre}\Big((\widehat{\sigma}^{\dagger a}_t(x))^2, \overline{C}, 1/\overline{C}\Big)$, where
$(\widehat{\sigma}^{\dagger a}_t(x))^2 = \widehat{\nu}^a_{t}(x) - \left(\widehat{\mu}^a_{t}(x)\right)^2$. Lastly, we construct the estimator of the target assignment ratio $\widehat{w}_t(a\mid x)$ by replacing $\sigma^a(x)$ by the estimator $\widehat{\sigma}^a_{t}(x)$. 

For obtaining estimators $\widehat{\mu}^a_{t}(x)$, $\widehat{\nu}^a_{t}(x)$, and $\widehat{\sigma}^a_{t}(x)$, we can use nonparametric estimators, such as the nearest neighbor regression estimator and kernel regression estimator, which have been proven to converge to the true function almost surely under a bounded sampling probability $\widehat{w}_t$ by \citet{yang2002} and \citet{qian2016kernel}. It should be noted that we do not assume specific convergence rates for estimators for $\mu^a(P)(x)$ and $w^*$ as the asymptotic optimality of the AIPW estimator can be demonstrated without them \citep{Laan2008TheCA,Kato2020adaptive,Kato2021adr}.

\subsection{Recommendation Rule with Policy Learning}
The following part presents our recommendation rule. To recommend the conditionally best treatment arm $a^*(P)(x)$,
we train a policy $\pi:[K]\times \mathcal{X} \to [0,1]$ by maximizing the empirically approximated policy value function, which we will describe below. 

At the end of an experiment, we estimate the policy value $Q(\pi)$ by using the augmented doubly robust estimator, which is defined as follows:
\begin{align}
    \label{eq:aipw}
    &\widehat{Q}_T(\pi) \coloneqq  \frac{1}{T}\sum^T_{t=1}\sum_{a\in[K]}\pi(a\mid  X_t)\widehat{\Gamma}^a_t,
\end{align}
where
\begin{align}
    \widehat{\Gamma}^a_t \coloneqq  \frac{\mathbbm{1}[A_t = a]\big(c_T\left(Y_t\right) - \widehat{\mu}^a_t(X_t)\big)}{\widehat{w}_t(a\mid  X_t)} + \widehat{\mu}^a_t(X_t),\quad c_T\left(Y_t\right) = \mathrm{thre}\Big(Y_t, U_T, -U_T\Big),
\end{align}
and $U_T$ is a positive value that approaches infinity as $T \to \infty$.
Then, we train a policy as
\begin{align}
\widehat{\pi}^{\mathrm{PLAS}}_T \coloneqq  \argmax_{\pi\in\Pi} \widehat{Q}_T(\pi)
\end{align}
By using this trained policy, given $x\in\mathcal{X}$, we recommend $\widehat{a}_T(x) \in [K]$ as the best treatment arm 
with probability $\widehat{\pi}^{\mathrm{PLAS}}_T\left(\widehat{a}_T(x)\mid x\right)$.

The AIPW estimator debiases the sample selection bias resulting from treatment assignment based on contextual information. Additionally, the AIPW estimator possesses the following properties: (i) its components $\{\widehat{\Gamma}^a_t - \mu^a(P)(x)\}^T_{t=1}$ are a martingale difference sequence, allowing us to employ the martingale limit theorems in derivation of the upper bound; (ii) it has the minimal asymptotic variance among the possible estimators. For example, other estimators with a martingale property, such as the inverse probability weighting (IPW) estimator, may be employed, yet their asymptotic variance would be greater than that of the AIPW estimator. The $t$-th element of the sum in the AIPW estimator utilizes nuisance parameters ($\mu^a(P)(x)$ and $w^*$) estimated from past observations up to round $t-1$ for constructing a martingale difference sequence \citep{Laan2008TheCA,hadad2019,Kato2020adaptive,Kato2021adr}. For those reasons, this estimator is often used in the context of adaptive experimental design.

\begin{algorithm}[tb]
   \caption{PLAS strategy}
   \label{alg}
\begin{algorithmic}
   \STATE {\bfseries Parameter:} Positive constants $C_{\mu}$ and $C_{\sigma^2}$.
   \STATE {\bfseries Initialization:} 
   \FOR{$t=1$ to $K$}
   \STATE Assign $A_t=t$. For each $a\in[K]$, set $\widehat{w}_{t}(a\mid x) = 1/K$.
   \ENDFOR
   \FOR{$t=K+1$ to $T$}
   \STATE Observe covariate $X_t$. 
   \STATE Construct the estimated target assignment ratio $\widehat{w}_{t}$ defined in Definition~\ref{def:est_target_assignment}.
   \STATE Draw $\xi_t$ from the uniform distribution on $[0,1]$. 
   \STATE $A_t = 1$ if $\xi_t \leq \widehat{w}_{t}(1\mid X_t)$ and $A_t = a$ for $a \geq 2$ if $\xi_t \in \left(\sum^{a-1}_{b=1}\widehat{w}_{t}(b\mid X_t), \sum^a_{b=1}\widehat{w}_{t}(b\mid X_t)\right]$.  
   \ENDFOR
   \STATE Construct $\widehat{Q}(\pi)$ following \eqref{eq:aipw}.
   \STATE Train a policy $\widehat{\pi}^{\mathrm{PLAS}}_T$ as $\widehat{\pi}^{\mathrm{PLAS}}_T = \argmax_{\pi\in\Pi} \widehat{Q}(\pi)$. 
   \STATE Recommend $\widehat{a}_T$ following $\widehat{\pi}^{\mathrm{PLAS}}_T$.
\end{algorithmic}
\end{algorithm} 

\section{Regret Upper Bound}
\label{sec:asymp_opt}
This section provides upper bounds for the expected simple regret of the PLAS strategy. First, we assume the following convergence rate for estimators of $\mu^a(P)(x)$ and $w^*(a\mid x)$. 
\begin{assumption}
\label{asm:consistent}
For any $\zeta$, any $P\in\mathcal{P}_\zeta$, and  all $a \in [K]$, it holds that
    \begin{align}
        &\sup_{x\in\mathcal{X}}\Big|\widehat{w}_T(a\mid  x) - w^*(a\mid  x) \Big|\xrightarrow{\mathrm{a.s.}} 0,\quad \sup_{x\in\mathcal{X}}\Big|\widehat{\mu}^a_T(x) - \mu^a(P)(x) \Big|\xrightarrow{\mathrm{a.s.}} 0\qquad \mathrm{as}\ \ T\to \infty.
    \end{align}
\end{assumption}

Next, we define an entropy integral of a policy class.
\begin{definition}
	Given the feature domain $\mathcal{X}$, a policy class $\Pi$, a set of $n$ points $\{x_1, \dots, x_n\} \subset \mathcal{X}$, define:
	\begin{enumerate}
		\item
		Hamming distance between any two policies $\pi_1$ and $\pi_2$ in $\Pi$: $H(\pi_1, \pi_2) = \frac{1}{n} \sum_{j=1}^n \mathbbm{1}\left[\pi_1(x_j) \neq \pi_2(x_j)\right]$. 
		\item 
		$\epsilon$-Hamming covering number of the set $\{x_1, \dots, x_n\}$:

			$\mathbb{N}_H(\epsilon, \Pi, \{x_1, \dots, x_n\})$ is the smallest number $K$ of policies $\{\pi_1, \dots, \pi_K\}$ in $\Pi$, such that $\forall \pi \in \Pi, \exists \pi_i, H(\pi, \pi_i) \le \epsilon$.
		
		\item
		$\epsilon$-Hamming covering number of $\Pi$:
		$\mathbb{N}_H(\epsilon, \Pi) = \sup\{\mathbb{N}_H(\epsilon, \Pi, \{x_1, \dots, x_m\}) \mid  m \ge 1,  x_1, \dots, x_m \in \mathcal{X}\}.$
		
		\item
		Entropy integral:
		$\kappa(\Pi) = \int_0^1 \sqrt{\log \mathbb{N}_H(\epsilon^2, \Pi)}d\epsilon$.
	\end{enumerate}
\end{definition}
The entropy represents the complexity of a policy class, as well as the Natarajan dimension. Between the entropy integral $\kappa(\Pi)$ and the Natarajan dimension $d_{\mathrm{N}}(\Pi)$, $\kappa(\Pi) \leq  C\sqrt{\log(d)d_{\mathrm{N}}(\Pi)}$ holds for some universal constant $C > 0$ \citep{jin2023upper,zhan2022policy} when $K \geq 3$. When $K = 2$, $d_{\mathrm{N}}(\Pi)$ is equal to the VC dimension, and $\kappa(\Pi) \leq  2.5\sqrt{d_{\mathrm{N}}(\Pi)}$ holds \citep{Haussler1995}.

Furthermore, we make the following assumption for the $\epsilon$-Hamming covering number.
\begin{assumption}\label{assump:1} For all $\epsilon \in (0, 1)$, $\mathbb{N}_H(\epsilon,\Pi) \le C\exp(D({\frac{1}{\epsilon}})^{\omega})$ for some constants $C,D > 0, 0 < \omega < 0.5$.
\end{assumption}

Then, we obtain the following upper bound for the expected simple regret of the PLAS strategy.
\begin{theorem}[Upper bound]
\label{thm:regret_upper_bound}
Suppose that Assumption~\ref{asm:consistent} holds. Then, for any $\zeta$ and any $P\in\mathcal{P}_\zeta$,  the expected simple regret of the PLAS strategy satisfies
\begin{align}
    &\sqrt{T}\mathbb{E}_P\Big[R\left(P\right)\left(\widehat{\pi}^{\mathrm{PLAS}}_T\right)\Big] \\
    &\leq \begin{cases}
        \Big(108.8 \kappa(\Pi) + 870.4\Big)\mathop{\mathbb{E}}_{X \sim \zeta}\left[\sqrt{\sum^K_{a=1}\left(\sigma^a(X)\right)^2}\right] + o(1) & \mathrm{if}\ K \geq 3\\
        \Big(108.8 \kappa(\Pi) + 870.4\Big)\mathop{\mathbb{E}}_{X \sim \zeta}\left[\sqrt{\left(\sigma^1(X) + \sigma^2(X)\right)^2}\right] + o(1) & \mathrm{if}\ K = 2
    \end{cases}\quad \mathrm{as}\ \ T\to \infty,
\end{align}
where $\kappa(\Pi) = 
    \begin{cases}
         2.5\sqrt{d_{\mathrm{N}}(\Pi)} & \mathrm{if}\ K = 2\\
         C\sqrt{\log(d)d_{\mathrm{N}}(\Pi)} & \mathrm{if}\ K \geq 3
    \end{cases}$ holds for a universal constant $C > 0$.
\end{theorem}

From Lemma~\ref{thm:regret_upper_bound} and the relationship between the entropy integral and the Natarajan dimension, the following theorem holds.

When $K = 2$, for $M = d_{\mathrm{VC}}(\Pi)$, the regret upper bound is given as $\Big(272 \sqrt{M} + 870.4\Big)\mathop{\mathbb{E}}_{X \sim \zeta}\left[\sqrt{\sum^K_{a=1}\left(\sigma^a(X)\right)^2}\right] + o(1)$ as $T \to \infty$. When $K \geq 3$, for $M = d_{\mathrm{N}}(\Pi)$, the regret upper bound is given as $\Big(108.8 C\sqrt{\log(d)M} + 870.4\Big)\mathop{\mathbb{E}}_{X \sim \zeta}\left[\sqrt{\sum^K_{a=1}\left(\sigma^a(X)\right)^2}\right] + o(1)$ as $T \to \infty$. Here, we note that the leading factor in these upper bounds are  $\mathop{\mathbb{E}}_{X \sim \zeta}\left[\sqrt{\log(d)M\sum^K_{a=1}\left(\sigma^a(X)\right)^2}\right]$ and $\mathop{\mathbb{E}}_{X \sim \zeta}\left[\sqrt{M\left(\sigma^a(X) + \sigma^2(X)\right)^2}\right]$, a product of the policy complexity ( Natarajan dimension $d_{\mathrm{N}}(\Pi)$) and outcome variances. This theorem implies that the leading factors align with the lower bounds with high probability. 

%Here, the terms $\sqrt{\sum^K_{a=1}\mathop{\mathbb{E}}_{X \sim \zeta}\left[\left(\sigma^a(X)\right)^2\right]}$ and $\sqrt{\mathop{\mathbb{E}}_{X \sim \zeta}\left[\left(\sigma^a(X) + \sigma^2(X)\right)^2\right]}$ remain, which are relatively smaller than the leading factors if $M$ is sufficiently large. 
%We might tight these terms as $\mathop{\mathbb{E}}_{X \sim \zeta}\left[\sqrt{\sum^K_{a=1}\left(\sigma^a(X)\right)^2}\right]$ and $\mathop{\mathbb{E}}_{X \sim \zeta}\left[\sqrt{\left(\sigma^a(X) + \sigma^2(X)\right)^2}\right]$, which is an open issue.

\begin{figure}[t!]
  \centering
  \includegraphics[width=140mm]{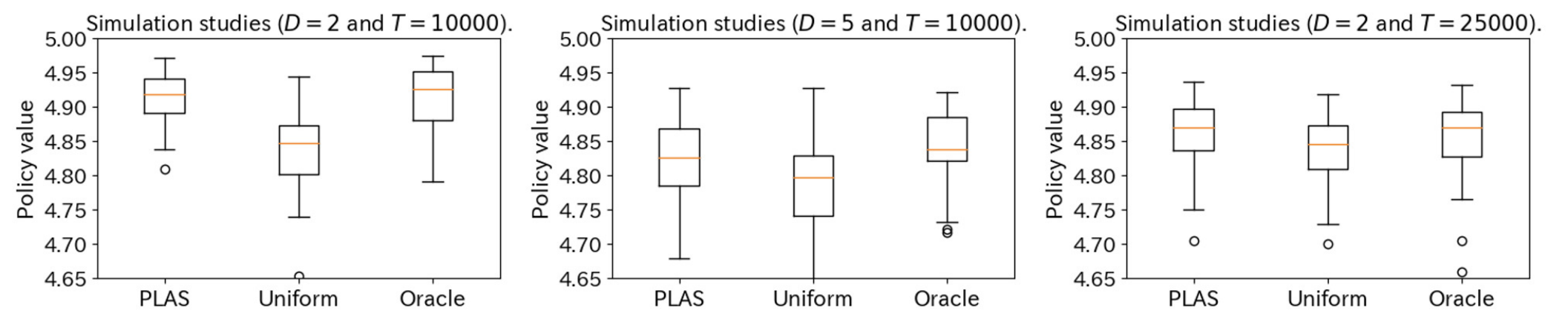}
  \caption{The results of simulation studies. The $y$-axis is the policy value of the learned policy.}
\label{fig:res1}
\vspace{-5mm}
\end{figure}

\section{Simulation Study}

We conduct simulation studies to investigate the empirical performance of our proposed PLAS strategy. We compare the PLAS strategy with a combination of uniform sampling and policy learning, denoted as Uniform. The Uniform strategy assigns treatment arms with an equal ratio of $1/K$ and then applies policy learning. As a baseline method, we use the PLAS strategy with known variances, referred to as Oracle.

We consider a simple scenario with $K = 4$. We examine three cases for $d$ and $T$: $(d, T) = (2, 10000)$, $(d, T) = (5, 10000)$, and $(d, T) = (5, 25000)$. Let $X_{i}$ be the $i$-th dimension of $X$, and let $m_{i}$ and $v_{i}$ be its mean and variance, respectively. The mean $m_{i}$ is drawn from a uniform distribution with support $[-1, 1]$, and the variance $v_{i}$ is fixed at $1$. If $X_{(1)} > 0.5$ and $X_{(2)} > 0.5$, then $\mu^1(P)(X) = 5.00$ and $\mu^2(P)(X) = \mu^3(P)(X) = \mu^4(P)(X) = 4.50$; if $X_{(1)} < 0.5$ and $X_{(2)} > 0.5$, then $\mu^2(P)(X) = 5.00$ and $\mu^1(P)(X) = \mu^3(P)(X) = \mu^4(P)(X) = 4.50$; if $X_{(1)} > 0.5$ and $X_{(2)} < 0.5$, then $\mu^2(P)(X) = 5.00$ and $\mu^1(P)(X) = \mu^3(P)(X) = \mu^4(P)(X) = 4.50$; if $X_{(1)} < 0.5$ and $X_{(2)} < 0.5$, then $\mu^4(P)(X) = 5.00$ and $\mu^1(P)(X) = \mu^2(P)(X) = \mu^3(P)(X) = 4.50$.

We conduct $50$ independent trials to evaluate the performance of the strategies. The results are presented in Figure~\ref{fig:res1} with three different settings for $d$ and $T$. The $y$-axis represents the policy value of the learned policy. From the experimental results, we confirm that our proposed strategy effectively improves the policy value.
\section{Conclusion}

In this study, we presented an adaptive experiment with policy learning. Our main contributions include the derivation of lower bounds for strategies, the development of the PLAS strategy, and the establishment of its regret upper bound. First, by utilizing the lower bounds developed by \citet{Kaufman2016complexity}, we derived lower bounds for the expected simple regret, which depend on the variances of outcomes. Then, based on these lower bounds, we developed the PLAS strategy, which trains a policy at the end of the experiment. Lastly, we provided upper bounds for the regret of the PLAS strategy.

From a technical perspective, we demonstrated how to use Rademacher complexity for i.i.d. samples in an adaptive experiment with non-i.i.d. samples. We did not employ complexity measures for non-i.i.d. samples as presented by \citet{Rakhlin2015} and \citet{Foster2023}. Instead, our technique relies on an approach used by \citet{Hahn2011}, which we extended by incorporating sample splitting, also known as double machine learning \citep{Laan2008TheCA,ZhengWenjing2011CTME,ChernozhukovVictor2018Dmlf,hadad2019,Kato2020adaptive,Kato2021adr}.

We also contributed to the literature on policy learning by providing a variance-dependent lower bound, which applies to observational studies with i.i.d. samples, and by discussing matching upper bounds. Our derived lower bound is distinct from that in \citet{AtheySusan2017EPL} and more tightly depends on variances, necessitating the refinement of existing upper bounds.

Our next step is to tighten both the lower and upper bounds. When $K=2$, we showed that assigning each treatment arm a proportion based on standard deviations is optimal, consistent with existing works such as \citet{Neyman1934OnTT}, \citet{glynn2004large}, and \citet{Kaufman2016complexity}. However, we found that assigning each treatment arm a proportion based on variances is optimal when $K \geq 3$. Other studies on fixed-budget BAI without contextual information, such as \citet{glynn2004large}, \citet{Kaufman2016complexity}, and \citet{kato2023locally}, indicate that strategies with different sampling rules are optimal. Bridging the gap between our study and these existing studies remains an open issue.

\bibliographystyle{icml2024}
\bibliography{output.bbl} 

\clearpage

\appendix

\section{Related Work}
\label{sec:related_work}
The MAB problem has been explored as an instance of the sequential decision-making problem \citep{Thompson1933, Robbins1952, Lai1985}, where BAI is a paradigm within this context \citep{EvanDar2006, Audibert2010, Bubeck2011}. 

\paragraph{BAI and ordinal optimization.}
The study of BAI can be traced back to sequential testing, ranking, and selection problems in the 1940s \citep{Wald1947, bechhofer1968sequential}. Subsequent studies in operations research, particularly in the realm of ordinal optimization, have garnered considerable attention \citep{chen2000, glynn2004large}. These studies focus on devising optimal strategies under the assumption of known distributional parameters. The machine learning community has reframed the problem as the BAI problem, placing a specific emphasis on estimating unknown distributions \citep{Jenninson1982,EvanDar2006,Audibert2010, Bubeck2011}.

\citet{Audibert2010} propose the UCB-E and Successive Rejects (SR) strategies. \citet{Bubeck2011} demonstrates minimax optimal strategies for expected simple regret in a non-asymptotic setting by extending the minimax lower bound of \citet{Auer2002}. \citet{Carpentier2016} further enhances the minimax lower bound, showing the optimality of \citet{Audibert2010}'s methods in terms of leading factors in expected simple regret. Based on their lower bound,\citet{yang2022minimax} proposes minimax optimal linear fixed-budget BAI.

In addition to minimax evaluation, \citet{Komiyama2021} develop an optimal strategy whose upper bound for simple Bayesian regret lower bound aligns with their derived lower bound. \citet{atsidakou2023bayesian} propose a Bayes optimal strategy for minimizing the probability of misidentification, revealing a surprising result that a $1/\sqrt{T}$-factor dominates the evaluation.

\citet{Russo2016}, \citet{Qin2017}, and \citet{Shang2020} propose Bayesian BAI strategies that are optimal in terms of posterior convergence rate. \citet{Kasy2021} and \citet{Ariu2021} discuss that such optimality does not necessarily guarantee asymptotic optimality for the probability of misidentification in fixed-budget BAI.

In contrast to the approaches of \citet{Bubeck2011} and \citet{Carpentier2016}, regarding asymptotic optimality, \citet{Kaufman2016complexity} derives distribution-dependent lower bounds for BAI with fixed confidence and a fixed budget, based on change-of-measure arguments and building upon the work of \citet{Lai1985}. Following their work, \citet{Garivier2016} proposes an optimal strategy for BAI with fixed confidence; however, in the fixed-budget setting, there is currently a lack of strategies whose upper bound matches the lower bound established by \citet{Kaufman2016complexity}. This issue has been discussed by \citet{kaufmann2020hdr}, \citet{Ariu2021}, \citet{Qin2022open}, \citet{degenne2023existence}, \citet{kato2023worstcase}, and \citet{wang2023uniformly}.

\citet{Kock2020} generalizes the results of \citet{Bubeck2011} for the case where the parameter of interest is a functional of the distribution and finds that, in contrast to the results \citet{Bubeck2011}, the target allocation ratio is not uniform.

The problem of BAI with contextual information is still under investigation. For example, \citet{Tekin2015}, \citet{GuanJiang2018}, \citet{Deshmukh2018}, \citet{Kato2021Role}, and \citet{Qin2022} consider this problem, but their analyses and settings differ from those employed in this study.

Our proposed strategy assigns treatment arms with a probability depending on their variances. Variance-dependent BAI has been explored by \citet{chen2000}, \citet{glynn2004large}, \citet{Kaufman2016complexity}, \citet{sauro2020rapidly}, \citet{Jourdan2023}, \citet{Kato2023minimax}, \citet{kato2023worstcase, kato2023locally,kato2024agna}, and \citet{lalitha2023fixedbudget}. Our choice of treatment-assignment probability is also inspired by \citet{Laan2008TheCA} and \citet{Hahn2011} in adaptive experimental design for efficient treatment effect estimation.

\paragraph{Decision theory and treatment choice.}
Beyond BAI, our study is further related to statistical decision theory \citep{Wald1949}. \citet{Manski2000, Manski2002, Manski2004} extend this decision theory and introduce the treatment choice problem from a decision theory perspective, independent of BAI. They focus on recommending the best treatment arm using non-experimental, independently and identically distributed (i.i.d.) observations without adaptive experimental design \citep{Schlag2007eleven, Stoye2009, Stoye2012, ManskiTetenov2016, Dominitz2017, Dominitz2022}. \citet{Hirano2009} employs the limit experiment framework \citep{LeCam1972, LeCam1986, LehmCase98, Vaart1991, Vaart1998} for discussing the problem of treatment choice, where the class of alternative hypotheses comprises local models, with parameters of interest converging to the true parameters at a rate of $1/\sqrt{T}$. \citet{Armstrong2022} and \citet{hirano2023asymptotic} apply this framework to adaptive experimental design. \citet{Adusumilli2021risk,adusumilli2022minimax} present an alternative minimax evaluation of bandit strategies for both regret minimization and BAI, based on a formulation utilizing a diffusion process proposed by \citet{Wager2021diffusion} and the limit experiment framework \citep{lecam1960locally, LeCam1972, LeCam1986, Vaart1991, Vaart1998}.

\paragraph{Policy learning.}
Inspired by supervised learning and statistical decision theory, various off-policy learning methods have been proposed. \citet{Swaminathan2015} proposes counterfactual risk minimization, an extension of empirical risk minimization for policy learning. \citet{Kitagawa2018} extends counterfactual risk minimization by linking it to the viewpoint of treatment choice and proposes welfare maximization. \citet{AtheySusan2017EPL} refines the method proposed by \citet{Kitagawa2018}. \citet{ZhouZhengyuan2018OMPL} shows a tight upper bound for general policy learning methods. \citet{zhan2022policy} develops a policy learning method from adaptively collected observations.

For the expected simple regret, \citet{Bubeck2011} shows that the Uniform-Empirical Best Arm (EBA) strategy is minimax optimal for bandit models with bounded supports. \citet{Kock2020} extends the results to cases where parameters of interest are functionals of the distribution and finds that optimal sampling rules are not uniform. \citet{adusumilli2022minimax, Adusumilli2021risk} consider a different minimax evaluation of bandit strategies for both regret minimization and BAI problems, based on a formulation using a diffusion process, as proposed by \citet{Wager2021diffusion}.

\paragraph{Other related work.}
Efficient estimation of ATE via adaptive experiments constitutes another area of related literature. \citet{Laan2008TheCA} and \citet{Hahn2011} propose experimental design methods to estimate ATE more efficiently by using covariate information in treatment assignments. \citet{Karlan2014} examine donors' responses to new information by applying the method of \citet{Hahn2011}. Subsequently, \citet{Meehan2022} and \citet{Kato2020adaptive} attempt to improve these studies, and more recently,  \citet{gupta2021efficient} proposes the use of instrumental variables in this context. \citet{viviano2022experimental} explores experimental designs for network inference.

We employ the Augmented Inverse Probability Weighting (AIPW) estimator in policy learning.
The AIPW estimator has been extensively used in the fields of causal inference and semiparametric inference \citep{Tsiatis2007semiparametric, BangRobins2005, ChernozhukovVictor2018Dmlf}. More recently, it has also been utilized in other MAB problems, as seen in \citet{Kim2021}, \citet{Ito2022}, \citet{Zimmert2019}, and \citet{Masoudian2021}.

\section{Proof of the Minimax Lower Bounds (Theorems \ref{thm:null_lower_bound} and \ref{thm:null_lower_bound_ref})}
\label{sec:proof_thms}
% In this section, we prove Theorems~\ref{thm:null_lower_bound} and \ref{thm:null_lower_bound_ref}. We first prove Theorem~\ref{thm:null_lower_bound}, lower bound that holds for $K \geq 2$.
% Next, Section~\ref{sec:proof_ref} shows a refined lower bound for $K = 2$. Note that the first lower bound given by Theorem~\ref{thm:null_lower_bound} also holds for $K = 2$, but the refined lower bound given by Theorem~\ref{thm:null_lower_bound_ref})} is tighter.

In this section, we establish the proofs for Theorems \ref{thm:null_lower_bound} and \ref{thm:null_lower_bound_ref}. Initially, we focus on proving Theorem \ref{thm:null_lower_bound}, which presents a lower bound applicable for cases where $K \geq 2$. Following this, Section \ref{sec:proof_ref} introduces a refined lower bound specifically for $K = 2$. While the initial lower bound from Theorem \ref{thm:null_lower_bound} is valid for $K = 2$, Theorem \ref{thm:null_lower_bound_ref} offers a tighter lower bound.

\subsection{Transportation Lemma}
\label{sec:transport}
Let $f^a_{P}(y^a\mid s)$ be a density of $Y^a$ conditional on $X=s$ under $P$. Let $\zeta_P(s)$ be a density of $X$ under $P$.

\citet{Kaufman2016complexity} derives the following result based on a change-of-measure argument, which is the principal tool in our lower bound.
Let us define a density of $(Y^1, Y^2, \dots, Y^K, X)$ under  a bandit model $P\in\mathcal{P}$ as
\begin{align*}
    p(y^1, y^2, \dots, y^K, s) = \prod_{a\in[K]} f^a_{P}(y^a\mid s)\zeta_P(s).
\end{align*}

Between two bandit models $P, Q \in \mathcal{P}$, following the proof of Lemma~1 in \citet{Kaufman2016complexity}, referred to as the \emph{transportation} lemma, we define the log-likelihood ratio of a sequence of observations $(X_t, Y_t, A_t)_{t=1}^{T}$ under a given strategy
% for a sequence of the sampling rule $(A_1, A_2, \dots, A_T)$ under a strategy
as
\begin{align*}
    L_T(P, Q) = \sum^T_{t=1} \sum_{a\in[K]}\mathbbm{1}[ A_t = a] \log \left(\frac{f^a_{P}(Y^a_{t}\mid  X_t)}{f^a_{Q}(Y^a_{t}\mid  X_t)}\right).
\end{align*}

As discussed by \citet{Kaufman2016complexity}, the transportation lemma immediately yields the following lemma.
\begin{lemma}[Lemma~1 and Remark~2 in \citet{Kaufman2016complexity}]
\label{prp:helinger}
Suppose that for any two bandit model $P,Q\in\mathcal{P}$ with $K$ treatment arms and for all $a\in[K]$, the distributions $P^a$ and $Q^a$ are mutually absolutely continuous, where $P^a$ and $Q^a$ are distributions of $(Y^a, X)$ under $P$ and $Q$, respectively. Then, for any $a\in[K]$ and $x\in \mathcal{X}$, any strategy satisfies
\begin{align*}
    \big| \mathbb{P}_P(\widehat{a}(x) = a) - \mathbb{P}_Q(\widehat{a}(x) = a)  \big| \leq \sqrt{\frac{\mathbb{E}_P\left[L_T(P, Q)\right]}{2}}
\end{align*}
\end{lemma}

\subsection{Restricted Bandit Models}
\label{sec:rest_bandit_model}
Fix any policy class $\Pi$ with Natarajan dimension $M$.
We can find a set of $M$ points $\mathcal{S} \coloneqq  \{s_1,s_2,\dots, s_M\}$ that are shattered by $\Pi$.
% To derive lower bounds under a policy class with the Natarajan dimension $M$, we find $M$ points $\mathcal{S} \coloneqq  (s_1,s_2,\dots, s_M)$ that are shattered by $\Pi$.
% Specifically, we consider submodels (restricted bandit models) $\mathcal{P}^\dagger$ of a local location-shift bandit model $\mathcal{P}$ defined as follows. 

\paragraph{Gaussian bandit models.}
We choose a specific $\zeta$, a marginal distribution on $X$, and define a specific subclass $\mathcal{P}^\dagger$ of a class of location-shift bandit models to derive a lower bound.
Let $\zeta$ be a distribution on $\mathcal{X}$ such that $\supp \zeta = \mathcal{S}$ and $\Pr_\zeta(X=s) = 1/M$ for any $s \in \mathcal{S}$.
We focus on Gaussian bandit models, where outcomes follow Gaussian distributions conditional on contexts. 
A subclass $\mathcal{P}^\dagger \subseteq \mathcal{P}$ is defined as follows:
\begin{align*}
    &\mathcal{P}^\dagger \coloneqq  \Bigg\{ P \in\mathcal{P} \colon
    % \mid \mathcal{X} = \mathcal{S},\ \
    \forall a \in [K]\ \ \forall s\in\mathcal{S}\ \left[ Y^a \mid X=s \sim \mathcal{N}\left(\mu^a(s), \left(\sigma^a(s)\right)^2\right),\ 
    \mu^a(s) \in \mathbb{R} \right],
    \ \mathrm{marg}_{\mathcal{X}}(P) = \zeta
    \Bigg\},
\end{align*}
where $\sigma^a \colon \mathcal{X} \to \mathbb{R}_+$ is given (introduced in Definition~\ref{def:ls_bc}).
%$P \in \mathcal P^\dagger$ iff $P \in \mathcal P^*_{\zeta_{\mathcal{S}}, \sigma}$ and, if $P$ is a null bandit model, then for all $s_m\in\mathcal{S}$
% Furthermore, following \citet{AtheySusan2017EPL}, for all $P\in \mathcal{P}^\dagger$ such that $\mu^1(P)(x) = \mu^2(P)(x) = \cdots = \mu^K(P)(x)$, we assume
%\[\sum^K_{a=1}\left(\sigma^a(s_m)\right)^2 = \mathbb{E}_P\left[\sum^K_{a=1}\left(\sigma^a(X)\right)^2\right]\]
%holds.

\paragraph{Alternative hypothesis}
The set of alternative hypotheses $\mathcal{Q}^\dagger \subseteq \mathcal{P}^\dagger$ is defined as follows:
\begin{align*}
    \mathcal{Q}^\dagger \coloneqq  \Bigg\{ P \in\mathcal{P}^\dagger \colon\ \ &
    \forall s \in \mathcal{S},\ \  m^s \in \mathbb{R}, \ d(s) \in [K], \  \Delta^{d(s)} > 0, \\
    &
    \forall s \in\mathcal{S},\ \ \mu^{d(s)}(P)(s) = m^s + \Delta^{d(s)}(s),\\
    &\forall s\in \mathcal{S},\ \ \forall b \in [K]\backslash \{d(s)\}\ \ \mu^b(P)(s) = m^s
    \Bigg\}.
\end{align*}
Note that a distribution in $\mathcal{Q}^\dagger$ is characterized by a parameter $(d, \Delta, m)$, where 
$d = (d(s))_s$, $\Delta = (\Delta^{d(s)})_s$, and $m=(m^s)_s$.
We denote a typical element of $\mathcal{Q}^\dagger$ by $Q_{d, \Delta, m}$.

Next, we will define a distribution $P_{d, \Delta, m}^{\sharp, s} \in \mathcal{P}^\dagger$ by
\begin{align}
    \forall a \in [K], \ &\mu^a\left( P_{d, \Delta, m}^{\sharp, s} \right)(s) = m^s, \\
    \forall s' \in \mathcal{S} \setminus \{s\} \forall a \in [K], \ &\mu^a\left( P_{d, \Delta, m}^{\sharp, s} \right)(s') = \mu^a\left( Q_{d, \Delta, m} \right)(s').
\end{align}
Note that, since $P_{d, \Delta, m}^{\sharp, s} \in \mathcal{P}^\dagger$, it is characterized once we fix the conditional means for each arm-context pair.
We write $P_{d, \Delta, m}^{\sharp, s}$ as $P^{\sharp, s}$ when $(d, \Delta, m)$ is clear from the context.

\paragraph{Change of measure.} 
% For each $\bm{d}$, we consider $Q_{d, \Delta, m} \in \mathcal{Q}^\dagger$ such that the following equation holds:
For any $Q_{d, \Delta, m} \in \mathcal{Q}^\dagger$ and for each $s\in\mathcal{S}$, the following equation holds:
\begin{align*}
    L_T(P^{\sharp, s}, Q_{d, \Delta, m}) & = \sum^T_{t=1}\sum_{a\in[K]}\left\{\mathbbm{1}[ A_t = a] \log \left(\frac{f^a_{P^{\sharp, s}}(Y^a_{t}\mid  X_t)}{f^a_{Q_{d, \Delta, m}}(Y^a_{t}\mid  X_t)}\right)\right\}\\
    & = \sum^T_{t=1}\sum_{a\in[K]}\sum_{s'\in\mathcal{S}}\left\{\mathbbm{1}[ A_t = a] \log \left(\frac{f^a_{P^{\sharp, s}}(Y^a_{t}\mid  s')}{f^a_{Q_{d, \Delta, m}}(Y^a_{t}\mid  s')}\right)\right\}\mathbbm{1}[X_t = s']\\
    & \overset{(i)}{=} \sum^T_{t=1}\sum_{a\in[K]}\left\{\mathbbm{1}[ A_t = a] \log \left(\frac{f^a_{P^{\sharp, s}}(Y^a_{t}\mid  s)}{f^a_{Q_{d, \Delta, m}}(Y^a_{t}\mid  s)}\right)\right\}\mathbbm{1}[X_t = s]\\
    & \overset{(ii)}{=}  \sum^T_{t=1}\left\{\mathbbm{1}[ A_t = d(s)] \log \left(\frac{f^{d(s)}_{P^{\sharp, s}}(Y^{d(s)}_{t}\mid  s)}{f^{d(s)}_{Q_{d, \Delta, m}}(Y^{d(s)}_{t}\mid  s)}\right)\right\}\mathbbm{1}[X_t = s].
\end{align*}
In $\overset{(i)}{=}$, we used $\frac{f^a_{P^{\sharp, s}}(Y^a_{t}\mid  s')}{f^a_{Q_{d, \Delta, m}}(Y^a_{t}\mid  s')} = 1$ for $s' \neq s$ from the definition of $P^{\sharp, s}$. In $\overset{(ii)}{=}$, we used the assumption that for each $b \in [K]\backslash \{d(s)\}$, it holds that $f^{b}_{P^{\sharp, s}}(y\mid  s) = f^{b}_{Q_{d, \Delta, m}}(y\mid  s)$. 

Given a strategy, its \emph{assignment ratio} $w \colon [K] \times \mathcal{X} \to [0,1]$ under $P$ is defined as follows:
\[
w(a \mid x) \coloneqq \mathbb{E}_P\left[ \sum_{t=1}^T \mathbbm{1}\{A_t = a\} \mid X=x \right].
\]
For submodel $\mathcal{Q}^\dagger$, the following lemma holds. 
\begin{lemma}
\label{lem:semipara}
For any $s\in\mathcal{S}$, $P^{\sharp, s}\in\mathcal{P}^\dagger$, $Q_{d, \Delta, m}\in \mathcal{Q}^\dagger$, and $T \in \mathbb{N}$, the following equality holds:
\begin{align}
    \mathbb{E}_{P^{\sharp, s}}\left[L_T(P^{\sharp, s}, Q_{d, \Delta, m})\right]
    =
     \frac{T}{2M}\frac{\left(\Delta^{d(s)}(s)\right)^2}{\frac{\left(\sigma^{d(s)}(s)\right)^2}{w(d(s)\mid s)}}.
\end{align}
\end{lemma}
\begin{proof}
We have
\begin{align*}
    &\mathbb{E}_{P^{\sharp, s}}\left[L_T(P^{\sharp, s}, Q_{d, \Delta, m})\right]\\
    &= \mathbb{E}_{P^{\sharp, s}}\left[\sum^T_{t=1}\left\{\mathbbm{1}[ A_t = d(s)] \log \left(\frac{f^{d(s)}_{P^{\sharp, s}}(Y^{d(s)}_{t}\mid  s)}{f^{d(s)}_{Q_{d, \Delta, m}}(Y^{d(s)}_{t}\mid  s)}\right)\right\}\mathbbm{1}[X_t = s]\right]\\
    &= \sum^T_{t=1}\mathbb{E}_{P^{\sharp, s}}\left[\mathbb{E}_{P^{\sharp, s}}\left[\left\{\mathbbm{1}[ A_t = d(s)] \log \left(\frac{f^{d(s)}_{P^{\sharp, s}}(Y^{d(s)}_{t}\mid  s)}{f^{d(s)}_{Q_{d, \Delta, m}}(Y^{d(s)}_{t}\mid  s)}\right)\right\}\mid X_t = s\right]\mathbbm{1}[X_t = s]\right]\\
    &= \sum^T_{t=1}\mathbb{E}_{P^{\sharp, s}}\left[w(d(s)\mid s) \mathbb{E}_{P^{\sharp, s}}\left[\log \left(\frac{f^{d(s)}_{P^{\sharp, s}}(Y^{d(s)}_{t}\mid  s)}{f^{d(s)}_{Q_{d, \Delta, m}}(Y^{d(s)}_{t}\mid  s)}\right)\mid X_t = s\right]\mathbbm{1}[X_t = s]\right]\\
    &= T w(d(s)\mid s)\frac{\left(\mu^{d(s)}(Q_{d, \Delta, m})(s) - \mu^{d(s)}(P^{\sharp, s})(s)\right)^2}{2\left(\sigma^{d(s)}(s)\right)^2} / M \\
    &=
    \frac{T}{2M} \frac{\left(\Delta^{d(s)}(s)\right)^2}{\left(\sigma^{d(s)}(s)\right)^2 / w(d(s)\mid s)}.
\end{align*}
To show the second last equality, we use
\[
\mathbb{E}_{P^{\sharp, s}}\left[\log \left(\frac{f^{d(s)}_{P^{\sharp, s}}(Y^{d(s)}_{t}\mid  s)}{f^{d(s)}_{Q_{d, \Delta, m}}(Y^{d(s)}_{t}\mid  s)}\right)\mid X_t = s\right] = \frac{\left(\mu^{d(s)}(Q_{d, \Delta, m})(s) - \mu^{d(s)}(P^{\sharp, s})(s)\right)^2}{2\left(\sigma^{d(s)}(s)\right)^2},
\]
which corresponds to a KL divergence between two Gaussian distributions.
\end{proof}

\subsection{Proof of Theorem~\ref{thm:null_lower_bound}: General Minimax Lower Bounds}
\label{appdx:gen_minimax}
%Let $L^a_T(P^{\sharp, s}, Q_{d, \Delta, m})$ be $\sum^T_{t=1}\left\{\mathbbm{1}[ A_t = d] \log \left(\frac{f^d_{P^{\sharp, s}}(Y^d_{t}| X_t)}{f^d_{\bm{\bm{\Delta}_{(\bm{d})}}}(Y^d_{t}| X_t)}\right)\right\}$. Under the class of bandit models, we show the following lemma. 
We derive the lower bound of the expected simple regret. 
\begin{proof}[Proof of Theorem~\ref{thm:null_lower_bound}]
Our proof below is built on the following line of reasoning:
First, suppose that nature selects a true distribution in a two-step process:
Initially, nature determines $(e^a(s))_{s,a} \in \mathbb{R}^{M \times K}_{+}$, ensuring that $\sum_a e^a(s)=1$ for each state $s \in \mathcal{S}$. It then selects the optimal arms $d(s)$ with probability $e^{d(s)}(s)$ for each state $s$. Subsequently, nature chooses some $(d, \Delta, m)$, determining $Q_{d, \Delta, m} \in \mathcal{Q}^\dagger$, which represents the true distribution.\footnote{
Note that the choice of $(m_s)_{s \in \mathcal{S}}$ does not affect any objects in the proof of Theorem~\ref{thm:null_lower_bound}.
}
We then focus on the expected simple regret that the decision-maker could encounter under these strategies of nature, regardless of the strategy employed by the decision-maker. 
By construction, there must exist
% $\bm{d}$, $\bm{\Delta}_{(\bm{d})}$, and
at least one distribution $Q_{d, \Delta, m} \in \mathcal{Q}^\dagger \subseteq \mathcal{P}$ that attains this regret. Hence, this value serves as a lower bound for the regret.

Fix any strategy of the decision-maker.
First, observe that, for each $P\in\mathcal{P}^\dagger$, the expected simple regret can be simplified as follows:
\begin{align}
    R(P)(\pi) &= \mathbb{E}_P\left[\sum_{a\in[K]}\left(\mu^{a^*(P)(X)}(P)(X) - \mu^a(P)(X)\right)\mathbb{P}_P\left(\widehat{a}_T(X) = a \right)\right]\\
    &= \sum_{s\in\mathcal{S}}\sum_{a\in[K]}\left(\mu^{a^*(P)(s)}(P)(s) - \mu^a(P)(s)\right)\mathbb{P}_P\left(\widehat{a}_T(s) = a\right)\mathbb{P}_P\left( X = s\right)\\
    &= \sum_{s\in\mathcal{S}}\sum_{a\in[K]}\left(\mu^{a^*(P)(s)}(P)(s) - \mu^a(P)(s)\right)\mathbb{P}_P\left(\widehat{a}_T(s) = a\right) / M.
\end{align}

Next, we decompose the expected simple regret by using the definition of $\mathcal{P}^\dagger$: fix any $(e^a(s))_{s,a}$ and $(d, \Delta, m)$. Then, the expected simple regret the decision-maker experiences is:
\begin{align}
&\sum_{s\in\mathcal{S}}\sum_{d(s)\in[K]}e^{d(s)}(s)\sum_{b\in[K]\backslash \{d(s)\}}\left(\mu^{d(s)}(Q_{d, \Delta, m})(s) - \mu^b(Q_{d, \Delta, m})(s)\right)\mathbb{P}_{Q_{d, \Delta, m}}\left(\widehat{a}_T(s) = b\right) / M \\
    % &=  \sum_{s\in\mathcal{S}}\sum_{d(s)\in[K]}e^{d(s)}(s)\sum_{b\in[K]\backslash \{d(s)\}}\left(\mu^{d(s)}(Q_{d, \Delta, m})(s) - \mu(P^{\sharp, s})(s)\right)\mathbb{P}_{Q_{d, \Delta, m}}\left(\widehat{a}_T(s) = b\right) / M\\
    &=   \sum_{s\in\mathcal{S}}\sum_{d(s)\in[K]}e^{d(s)}(s)\left\{\sum_{b\in[K]\backslash \{d(s)\}}\Delta^{d(s)}(s)\mathbb{P}_{Q_{d, \Delta, m}}\left(\widehat{a}_T(s) = b\right)\right\} / M\\
    &=  \sum_{s\in\mathcal{S}}\sum_{d(s)\in[K]}e^{d(s)}(s)\left\{\Delta^{d(s)}(s)\mathbb{P}_{Q_{d, \Delta, m}}\left(\widehat{a}_T(s) \neq d(s)\right)\right\} / M\\
    &=  \sum_{s\in\mathcal{S}}\sum_{d(s)\in[K]}e^{d(s)}(s)\left\{\Delta^{d(s)}(s)\left(1-\mathbb{P}_{Q_{d, \Delta, m}}\left(\widehat{a}_T(s) = d(s)\right)\right)\right\} / M, \label{eq:simple_regret_intermediate}
\end{align}
Note that, as we saw in the beginning of the proof, the regret lower bound is bounded from below by $\eqref{eq:simple_regret_intermediate}$.

From Lemma~\ref{prp:helinger}. and the definition of null consistent strategies, we have
\begin{align}\eqref{eq:simple_regret_intermediate}&=\sum_{s\in\mathcal{S}}\sum_{d(s)\in[K]}e^{d(s)}(s)\Bigg\{\Delta^{d(s)}\left(1-  \mathbb{P}_{P^{\sharp, s}}\left(\widehat{a}_T(s) = d(s)\right) +
    \mathbb{P}_{P^{\sharp, s}}\left(\widehat{a}_T(s) = d(s)\right) -
    \mathbb{P}_{Q_{d, \Delta, m}}\left(\widehat{a}_T(s) = d(s)\right)\right)
    \Bigg\} / M\\ 
    &\geq \sum_{s\in\mathcal{S}}\sum_{d(s)\in[K]}e^{d(s)}(s)\left\{\Delta^{d(s)}(s)\left\{1 - \mathbb{P}_{P^{\sharp, s}}\left(\widehat{a}_T(s) = d(s)\right) - \sqrt{\frac{\mathbb{E}_{P^{\sharp, s}}\left[L^{d(s)}_T(P^{\sharp, s}, Q_{d, \Delta, m})\right]}{2}}\right\}\right\} / M. \label{eq:simple_regret_intermediate_2}
\end{align}
% \begin{align*}
%     &\sup_{\stackrel{\Delta^{d(s)}(s)\in(0, \infty)}{\mathrm{for}\ s\in\mathcal{S}}}\sum_{s\in\mathcal{S}}\sum_{d(s)\in[K]}e^{d(s)}(s)\left\{\Delta^{d(s)}(s)\left(1-\mathbb{P}_{Q_{d, \Delta, m}}\left(\widehat{a}_T(s) = d(s)\right)\right)\right\} / M\\
%     &=\sup_{\stackrel{\Delta^{d(s)}(s)\in(0, \infty)}{\mathrm{for}\ s\in\mathcal{S}}}\sum_{s\in\mathcal{S}}\sum_{d(s)\in[K]}e^{d(s)}(s)\Bigg\{\Delta^{d(s)}\left(1-  \mathbb{P}_{P^{\sharp, s}}\left(\widehat{a}_T(s) = d(s)\right) +  \mathbb{P}_{P^{\sharp, s}}\left(\widehat{a}_T(s) = d(s)\right) -  \mathbb{P}_{Q_{d, \Delta, m}}\left(\widehat{a}_T(s) = d(s)\right)\right)\Bigg\} / M\\ 
%     &\geq \sup_{\stackrel{\Delta^{d(s)}(s)\in(0, \infty)}{\mathrm{for}\ s\in\mathcal{S}}}\sum_{s\in\mathcal{S}}\sum_{d(s)\in[K]}e^{d(s)}(s)\left\{\Delta^{d(s)}(s)\left\{1 - \mathbb{P}_{P^{\sharp, s}}\left(\widehat{a}_T(s) = d(s)\right) - \sqrt{\frac{\mathbb{E}_{P^{\sharp, s}}\left[L^{d(s)}_T(P^{\sharp, s}, Q_{d, \Delta, m})\right]}{2}}\right\}\right\} / M.
% \end{align*}

Since we assume that the strategy is null consistent, we have $\mathbb{P}_{P^{\sharp, s}}\left(\widehat{a}_T(s) = d(s)\right) = 1/K + o(1)$.
By Lemma~\ref{lem:semipara}, we obtain
\begin{align}
\eqref{eq:simple_regret_intermediate_2}
    =\sum_{s\in\mathcal{S}}\sum_{d(s)\in[K]}e^{d(s)}(s)\Delta^{d(s)}(s) \left\{ 1 - \frac{1}{K} - \sqrt{\frac{T(\Delta^{d(s)}(s))^2}{4M \frac{\left(\sigma^{d(s)}(s)\right)^2}{w(d(s)\mid s)}}} \right\}/M + o(1).
\end{align}

% \begin{align}
%     &R(P)\geq \inf_{w\in\mathcal{W}}\sup_{\stackrel{\Delta^{d(s)}(s)\in(0, \infty)}{\mathrm{for}\ s\in\mathcal{S}}}\sum_{s\in\mathcal{S}}\sum_{d(s)\in[K]}e^{d(s)}(s)\Delta^{d(s)}(s) \left\{ 1 - \frac{1}{K} - \sqrt{\frac{T(\Delta^{d(s)}(s))^2}{4M \frac{\left(\sigma^{d(s)}(s)\right)^2}{w(d(s)\mid s)}}} \right\}/M.
% \end{align}

Let
\[
\mathcal{E} \coloneqq \left\{
(e^d(s))_{d,s} \in \mathbb{R}^{[K] \times \mathcal{S}} \colon
e^d(s) \in [0,1], \ \forall s \in \mathcal{S}, \sum_{d \in [K]} e^d(s) =1
\right\},
\]
and denote its typical element by $\bm{e}$.
In principle, if we take the supremum with respect to $\bm{e} \in \mathcal{E}$ and $(\Delta^{d(s)}(s))_s \in \mathbb{R}_{>0}^{M}$ in \eqref{eq:simple_regret_intermediate_2}, that will be a regret lower bound.
By substituting $\Delta^{d(s)}(s) = \sqrt{M{\frac{\left(\sigma^{d(s)}(s)\right)^2}{w(d(s)\mid s)}}/{2T}}$, taking the infimum with respect to $w$, and taking the supremum with respect to $(e^a(s))_{a,s}$, we obtain the following regret lower bound, which is lower than the lower bound obtained by taking the supremum with respect to $\bm{e}$ and $\Delta$:
\begin{align}
    % R(P) &\geq
    &\frac{1}{8}
    \frac{1}{\sqrt{MT}}
    \sum_{s\in\mathcal{S}}
    \sup_{\bm{e} \in \mathcal{E}}
    % \sup_{\stackrel{(e^1(s), e^2(s),\dots, e^K(s))\in[0, 1]^K}{\mathrm{s.t.}\ \sum_{a\in[K]}e^a(s)=1}}
    \inf_{\bm{w} \in \mathcal{W}}\sum_{d(s)\in[K]}e^{d(s)}(s)\left\{\sqrt{{\frac{\left(\sigma^{d(s)}(s)\right)^2}{w(d(s)\mid s)}}}\right\} + o(1).
\label{eq:simple_regret_intermediate_3}
\end{align}

Fix any $s \in \mathcal{S}$.
Let $h(\bm{e},\bm{w}) \coloneqq \sum_{d\in[K]}e^d(s)\sqrt{{\frac{\left(\sigma^{d}(s)\right)^2}{w(d\mid s)}}}$ and consider
$
\sup_{\bm{e} \in \mathcal{E}}
\inf_{\bm{w} \in \mathcal{W}} h(\bm{e},\bm{w})$.
Note that $h$ is concave in $\bm{e}$, convex in $\bm{w}$, and continuous in $(\bm{e}, \bm{w})$. Denote the closure of $\mathcal{W}$ by $\overline{\mathcal{W}}$.
First, observe that
\[
\sup_{\bm{e} \in \mathcal{E}}
\inf_{\bm{w} \in \mathcal{W}} h(\bm{e},\bm{w})
\geq
\sup_{\bm{e} \in \mathcal{E}}
\inf_{\bm{w} \in \overline{\mathcal{W}}} h(\bm{e},\bm{w}).
\]
Since $\overline{\mathcal{W}}$ is compact and $h(\bm{e}, \cdot)$ is continuous for each $\bm{e}$, we have
\[
\sup_{\bm{e} \in \mathcal{E}}
\inf_{\bm{w} \in \overline{\mathcal{W}}} h(\bm{e},\bm{w})
=
\sup_{\bm{e} \in \mathcal{E}}
\min_{\bm{w} \in \overline{\mathcal{W}}} h(\bm{e},\bm{w}).
\]
By Berge's Maximum Theorem, $\min_{\bm{w} \in \overline{\mathcal{W}}} h(\bm{e},\bm{w})$ is continuous in $\bm{w}$. Since $\mathcal{E}$ is compact, we have
\[
\sup_{\bm{e} \in \mathcal{E}}
\min_{\bm{w} \in \overline{\mathcal{W}}} h(\bm{e},\bm{w})
=
\max_{\bm{e} \in \mathcal{E}}
\min_{\bm{w} \in \overline{\mathcal{W}}} h(\bm{e},\bm{w}).
\]
Since $\mathcal{E}$ and $\overline{\mathcal{W}}$ are both compact and convex and $h$ is concave-convex, by the minimax theorem, we have
\[
\max_{\bm{e} \in \mathcal{E}}
\min_{\bm{w} \in \overline{\mathcal{W}}} h(\bm{e},\bm{w})
=
\min_{\bm{w} \in \overline{\mathcal{W}}} 
\max_{\bm{e} \in \mathcal{E}}
h(\bm{e},\bm{w}).
\]
Combining these results, we have
\begin{align}
    \sup_{\bm{e} \in \mathcal{E}}
\inf_{\bm{w} \in \mathcal{W}} h(\bm{e},\bm{w})
\geq
\min_{\bm{w} \in \overline{\mathcal{W}}} 
\max_{\bm{e} \in \mathcal{E}}
h(\bm{e},\bm{w}).
\end{align}
% \begin{align}
%     \sup_{\bm{e} \in \mathcal{E}}
%     \inf_{\bm{w} \in \mathcal{W}} h(\bm{e},\bm{w})
%     &\geq
%     \sup_{\bm{e} \in \mathcal{E}}
%     \inf_{\bm{w} \in \overline{\mathcal{W}}} h(\bm{e},\bm{w}) \\
%     &=
%     \inf_{\bm{w} \in \overline{\mathcal{W}}}
%     \sup_{\bm{e} \in \mathcal{E}}
%      h(\bm{e},\bm{w}),
% \end{align}
% where the last equality follows from the minimax theorem.

% First, because $\sum_{d\in[K]}e^d(s)\sqrt{{\frac{\left(\sigma^{d}(s)\right)^2}{w(d\mid s)}}}$ is convex with respect to $e^d(s)$ and concave with respect to $w$, we have 
% \[\sup_{\stackrel{(e^1(s), e^2(s),\dots, e^K(s))\in[0, 1]^K}{\mathrm{s.t.}\ \sum_{a\in[K]}e^a(s)=1}}\inf_{\bm{w} \in \mathcal{W}}\sum_{d\in[K]}e^d(s)\sqrt{{\frac{\left(\sigma^{d}(s)\right)^2}{w(d\mid s)}}} = \inf_{\bm{w} \in \mathcal{W}}e^d(s)\sup_{\stackrel{(e^1(s), e^2(s),\dots, e^K(s))\in[0, 1]^K}{\mathrm{s.t.}\ \sum_{a\in[K]}e^a(s)=1}}\sum_{d\in[K]}\sqrt{{\frac{\left(\sigma^{d}(s)\right)^2}{w(d\mid s)}}}.\]

Let us consider $
\max_{\bm{e} \in \mathcal{E}}
h(\bm{e},\bm{w})$ for a fixed $w$.
At the optimum, we have
$e^d(s) = 1$ iff
$d \in \argmax_{d \in [K]} \sqrt{{\frac{\left(\sigma^{d}(s)\right)^2}{w(d\mid s)}}}$. Thus, we have
\begin{align}
    \min_{\bm{w} \in \overline{\mathcal{W}}} 
\max_{\bm{e} \in \mathcal{E}}
h(\bm{e},\bm{w})
=
\min_{\bm{w} \in \overline{\mathcal{W}}} 
\max_{d \in [K]}
\sqrt{{\frac{\left(\sigma^{d}(s)\right)^2}{w(d\mid s)}}}.
\label{eq:opt_lbd}
\end{align}

% In the tight lower bound, $e^{\widetilde{d}}(s) = 1$ for $\widetilde{d} = \argmax_{d\in[K]} \frac{1}{12}\sqrt{\frac{\left(\sigma^{d}(s)\right)^2}{w(d\mid s)}}$\footnote{If there are multiple candidates of the best treatment arm, we choose one of the multiple treatment arms as the best treatment arm with probability $1$.}. Therefore, we consider solving
% \begin{align*}
% \inf_{w\in\mathcal{W}}\max_{d\in[K]}\sqrt{\frac{\left(\sigma^{d(s)}\right)^2}{w(d\mid s)}}.
% \end{align*}
% If there exists a solution, we can replace the $\inf$ with the $\min$.
For each $s$, we consider the following constrained optimization:
% \begin{align}
% \label{eq:opt_prob}
%     \inf_{\stackrel{R\in \mathbb{R},\ \ \ \ (w(1\mid s)\ \ w(2\mid s) \ \ \dots\ \ w(K\mid s))^\top\in(0, 1)^K}{\mathrm{s.t.}\ \sum_{a\in[K]w(a\mid s)}}}&\quad R\\
%     \mathrm{s.t.} &\quad R \geq \frac{\left(\sigma^{d}(s)\right)^2}{w(d\mid s)}\quad \forall d\in[K]\nonumber\\
%     &\quad \sum_{a\in[K]}w(a\mid s) = 1.\nonumber
% \end{align}
\begin{align}
\label{eq:opt_prob}
    \min_{R \geq 0,
    w \in \overline{\mathcal{W}}
    }
    &\quad R
    \\
    \mathrm{s.t.} &\quad R \geq \sqrt{\frac{\left(\sigma^{d}(s)\right)^2}{w(d\mid s)}}\quad \forall d\in[K]\nonumber\\
    % &\quad \sum_{a\in[K]}w(a\mid s) = 1.\nonumber
\end{align}
Note that if $(R^*, w^*)$ is the optimal solution to this problem, $w^*$ is also the optimal solution to $\eqref{eq:opt_lbd}$.
After some algebra, we can show that
\[
w^*(d \mid s) = \frac{(\sigma^d(s))^2}{\sum_{a \in [K]} (\sigma^a(s))^2}
\]
for each $d \in [K]$.

% For this problem, we derive the first-order condition, which is sufficient for the global optimality of such a convex programming problem. For Lagrangian multipliers $\lambda^{d}\in (-\infty, 0]$ and $\gamma\in\mathbb{R}$, we consider the following Lagrangian function: 
% \begin{align*}
%     L(\bm{\lambda}, \gamma; R, \bm{w}) = R + \sum_{d\in [K]} \lambda^{d}\left\{\frac{\left(\sigma^{d}(s)\right)^2}{w(d\mid s)} - R\right\} + \gamma\left\{\sum_{d\in[K]}w(d\mid s) - 1\right\},
% \end{align*}
% where $\bm{w} \coloneqq  (w(1\mid s)\ \ w(2\mid s) \ \ \dots\ \ w(K\mid s))^\top\in[0, 1]^K$. 
% Then, the optimal solutions $\bm{w}^*$, $\lambda^{* d}$,  $\gamma^*$, and $R^*$ of the original problem satisfies
% \begin{align}
% \label{eq:opt_sol_const}
% &1 -  \sum_{d\in[K]}\lambda^{d *} = 0,\\
% &-\lambda^{d*}\frac{\left(\sigma^d(s)\right)^2}{(w^*(d\mid s))^2}  = \gamma^*\qquad \forall d \in [K],\\
% &\lambda^{d*} \left\{\frac{\left(\sigma^{d}(s)\right)^2}{w(d\mid s)} - R^*\right\} = 0\qquad \forall d \in [K],\\
% &\gamma^* \left\{\sum_{a\in[K]}w^*(a\mid s) - 1\right\} = 0.\nonumber
% \end{align}
% Here, the solutions are given as
% \begin{align*}
% &w^*(d\mid s) = \frac{\left(\sigma^d(s)\right)^2}{\sum_{b\in[K]}\left(\sigma^{b}(s)\right)^2},\\
% &\lambda^{d*} = w^*(d\mid s),\\
% &\gamma^* = - \sum_{b\in[K]}\left(\sigma^{b}(s)\right)^2.
% \end{align*}

Thus, we have
\begin{align}
    \eqref{eq:opt_lbd}
    =
    \sqrt{\sum_{a \in [K]} (\sigma^a(s))^2},
\end{align}
and hence
\begin{align}
\eqref{eq:simple_regret_intermediate_3}
    &\geq
    \frac{1}{8} \frac{1}{\sqrt{TM}} \sum_{s \in \mathcal{S}}
    \sqrt{ \sum_{a \in [K]} (\sigma^a(s))^2} 
    =
    \frac{1}{8} \frac{1}{\sqrt{T}} \sum_{s \in \mathcal{S}}
    \sqrt{M \sum_{a \in [K]} (\sigma^a(s))^2} \frac{1}{M} + o(1) \\
    &=
    \frac{1}{8} \frac{1}{\sqrt{T}} \mathop{\mathbb{E}}_{X \sim \zeta} \left[
    \sqrt{ M \sum_{a \in [K]} (\sigma^a(s))^2}
    \right] + o(1)
    .
    \label{eq:reg_lbd_final}
\end{align}
The last equality follows since we choose $\xi$ so that $\xi$ puts equal probabilities on $\mathcal{S}$.
This implies that, for any strategy of the decision-maker, there exists a distribution in $\mathcal{P}$ under which the simple regret is lower bounded by \eqref{eq:reg_lbd_final}.

% Therefore, we obtain
% \begin{align*}
%     \sup_{\stackrel{(e^1(s), e^2(s),\dots, e^K(s))\in[0, 1]^K}{\mathrm{s.t.}\ \sum_{a\in[K]}e^a(s)=1}}\inf_{\bm{w} \in \mathcal{W}}\sum_{a\in[K]}e^a(s) \frac{1}{10}\sqrt{M\frac{\left(\sigma^{a}(s)\right)^2}{w(a\mid s)}}& = \frac{1}{10} \sqrt{M\sum_{a\in[K]}\left(\sigma^{a}(s)\right)^2}\sum_{a\in[K]}e^{a}(s).
% \end{align*}
% Since $\sum_{a\in[K]}e^{a}(s) = 1$,  we have
% \begin{align}
%     \sqrt{T}R(P) \geq \frac{1}{10}\mathop{\mathbb{E}}_{X \sim \zeta}\left[ \sqrt{M\sum_{a\in[K]}\left(\sigma^{a}(X)\right)^2}\right].
% \end{align}
% Thus, the proof is complete. 

\end{proof}
% Here, $w^*(a\mid x) = \frac{\left(\sigma^{a}(x)\right)^2}{\sum_{b\in[K]}\left(\sigma^{b}(x)\right)^2}$ works as a target assignment ratio in implementation of our BAI strategy because it represents the sample average of $\mathbbm{1}[A_t = a]$; that is, we design our sampling rule $(A_t)_{t\in[T]}$ for the average to be the target assignment ratio. 

Although this lower bound is applicable to a case with $K = 2$, we can tighten the lower bound by changing the definiton of the parametric submodel.

\subsection{Refined Minimax Lower Bounds for Two-armed Bandits (Proof of Theorem~\ref{thm:null_lower_bound_ref})}
\label{sec:proof_ref}
When $K = 2$, we can derive a tighter lower bound. As in previous sections, we choose $\xi \in \mathcal{M}(\mathcal{X})$ such that $\Pr_\xi(X=s) = 1/M$ for each $s \in \mathcal{S}$.

\paragraph{Change of measure.} 
% For each $\bm{d}$, we consider $Q_{d, \Delta, m} \in \mathcal{Q}^\dagger$ such that the following equation holds: for each $s\in\mathcal{S}$,
For any $Q_{d, \Delta, m} \in \mathcal{Q}^\dagger$ and $s \in \mathcal{S}$, the following equation holds:
\begin{align*}
    L_T(P^{\sharp, s}, Q_{d, \Delta, m})
    &= \sum^T_{t=1}\sum_{a\in[K]}\left\{\mathbbm{1}[ A_t = a] \log \left(\frac{f^a_{P^{\sharp, s}}(Y^a_{t}\mid  X_t)}{f^a_{Q_{d, \Delta, m}}(Y^a_{t}\mid  X_t)}\right)\right\}\\
    &= \sum^T_{t=1}\sum_{a\in[K]}\sum_{s'\in\mathcal{S}}\left\{\mathbbm{1}[ A_t = a] \log \left(\frac{f^a_{P^{\sharp, s}}(Y^a_{t}\mid  s')}{f^a_{Q_{d, \Delta, m}}(Y^a_{t}\mid  s')}\right)\right\}\mathbbm{1}[X_t = s']\\
    & \overset{(i)}{=} \sum^T_{t=1}\sum_{a\in[K]}\left\{\mathbbm{1}[ A_t = a] \log \left(\frac{f^a_{P^{\sharp, s}}(Y^a_{t}\mid  s)}{f^a_{Q_{d, \Delta, m}}(Y^a_{t}\mid  s)}\right)\right\}\mathbbm{1}[X_t = s].
\end{align*}
In $\overset{(i)}{=}$, we used $\frac{f^a_{P^{\sharp, s}}(Y^a_{t}\mid  s')}{f^a_{Q_{d, \Delta, m}}(Y^a_{t}\mid  s')} = 1$ for $s' \neq s$ from the definition of $P^{\sharp, s}$.

\begin{proof}[Proof of Theorem~\ref{thm:null_lower_bound_ref}]

By the same argument as in Section~\ref{appdx:gen_minimax}, we have the following lower bound for the simple regret (cf. equation~\eqref{eq:simple_regret_intermediate_2}):
\begin{align}\sum_{s\in\mathcal{S}}\sum_{d(s)\in[K]}e^{d(s)}(s)\left\{\Delta^{d(s)}(s)\left\{1 - \mathbb{P}_{P^{\sharp, s}}\left(\widehat{a}_T(s) = d(s)\right) - \sqrt{\frac{\mathbb{E}_{P^{\sharp, s}}\left[L^{d(s)}_T(P^{\sharp, s}, Q_{d, \Delta, m})\right]}{2}}\right\}\right\} / M.
    \label{eq:lbd_1}
\end{align}

By the same argument as in the proof of Lemma~\ref{lem:semipara}, we have
\begin{align*}
    \mathbb{E}_{P^{\sharp, s}}\left[L_T(P^{\sharp, s}, Q_{d, \Delta, m})\right]
    &= \frac{T}{2M} \Bigg\{w(1\mid s)\frac{\left(\mu^{1}(Q_{d, \Delta, m})(s) - m^s\right)^2}{\left(\sigma^{1}(s)\right)^2}  + w(2\mid s)\frac{\left(\mu^{2}(Q_{d, \Delta, m})(s) - m^s\right)^2}{\left(\sigma^{2}(s)\right)^2} \Bigg\}.
\end{align*}
% where we used that $\mathbb{E}_{P^{\sharp, s}}\left[\log \left(\frac{f^{a}_{P^{\sharp, s}}(Y^{a}_{t}\mid  s)}{f^{a}_{Q_{d, \Delta, m}}(Y^{a}_{t}\mid  s)}\right)\mid X_t = s\right] = \frac{\left(\mu^{a}(Q_{d, \Delta, m})(s) - \mu^{a}(P^{\sharp, s})(s)\right)^2}{2\left(\sigma^{a}(s)\right)^2}$ because it corresponds to a KL divergence between two Gaussian distributions.

We consider the following optimization problem with respect to $m^s$:
\[
\min_{m_s \in \mathbb{R}}
\Bigg\{w(1\mid s)\frac{\left(\mu^{1}(Q_{d, \Delta, m})(s) - m^s\right)^2}{\left(\sigma^{1}(s)\right)^2}  + w(2\mid s)\frac{\left(\mu^{2}(Q_{d, \Delta, m})(s) - m^s\right)^2}{\left(\sigma^{2}(s)\right)^2} \Bigg\}.
\]
The solution is
\begin{align}
    m^s = \frac{c^1\mu^{1}(Q_{d, \Delta, m})(s) + c^2\mu^{2}(Q_{d, \Delta, m})(s)}{c^1 + c^2},
\end{align}
where
\begin{align}
    c^a = \frac{w(a\mid s)}{\left(\sigma^{a}(s)\right)^2},
\end{align}
and, the optimal value is
\begin{align}
    \left(\Delta^{d(s)} \right)^2 \left[ \frac{(\sigma^1(s))^2}{w(1 \mid s)} + \frac{(\sigma^2(s))^2}{w(2 \mid s)} \right]^{-1}. \label{eq:lbd_opt_k=2}
\end{align}
% \begin{align}
     % \mathbb{E}_{P^{\sharp, s}}\left[L_T(P^{\sharp, s}, Q_{d, \Delta, m})\right]&= 
     % \frac{T\left(\mu^{1}(Q_{d, \Delta, m})(s) - \mu^{2}(Q_{d, \Delta, m})(s)\right)^2}{2M\Bigg\{\left(\sigma^{1}(s)\right)^2 / w(1\mid s) + \left(\sigma^{2}(s)\right)^2 / w(2\mid s)\Bigg\}}\\
     % &=
     % \frac{T\Delta^2(s)}{2M\Bigg\{\left(\sigma^{1}(s)\right)^2 / w(1\mid s) + \left(\sigma^{2}(s)\right)^2 / w(2\mid s)\Bigg\}}.
% \end{align}

Since we assume a null consistent strategy, we have $\mathbb{P}_{P^{\sharp, s}}\left(\widehat{a}_T(s) = d(s)\right) = 1/K + o(1)$.
By \eqref{eq:lbd_1} and \eqref{eq:lbd_opt_k=2}, we have the following regret lower bound:
\begin{align}
\inf_{w\in\mathcal{W}}\sup_{\stackrel{\Delta^{d(s)}\in(0, \infty)}{\mathrm{for}\ s\in\mathcal{S}}}\sum_{s\in\mathcal{S}}\sum_{d(s)\in[K]}e^{d(s)}(s)\Delta^{d(s)} \left\{ 1 - \frac{1}{K} - \sqrt{\frac{T\Delta^2(s)}{4M\Bigg\{\frac{\left(\sigma^{1}(s)\right)^2}{w(1\mid s)} + \frac{\left(\sigma^{2}(s)\right)^2}{w(2\mid s)}\Bigg\}}} \right\}/M
+o(1).
\end{align}

By substituting $\Delta^{d(s)} = \sqrt{M\Bigg\{\frac{\left(\sigma^{1}(s)\right)^2}{w(1\mid s)} + \frac{\left(\sigma^{2}(s)\right)^2}{w(2\mid s)}\Bigg\}/2 T}$, we obtain the following regret lower bound:
\begin{align}
    &\frac{1}{8}\frac{1}{\sqrt{MT}}
    \sum_{s\in\mathcal{S}}\sup_{e \in \mathcal{E}}\inf_{\bm{w} \in \mathcal{W}}\sum_{d(s)\in[K]}e^{d(s)}(s)\left\{\sqrt{\frac{\left(\sigma^{1}(s)\right)^2}{w(1\mid s)} + \frac{\left(\sigma^{2}(s)\right)^2}{w(2\mid s)}}\right\} +o(1) \\
    &=
    \frac{1}{8}\frac{1}{\sqrt{MT}}
    \sum_{s\in\mathcal{S}}\sup_{e \in \mathcal{E}}\inf_{\bm{w} \in \mathcal{W}}
    \left\{\sqrt{\frac{\left(\sigma^{1}(s)\right)^2}{w(1\mid s)} + \frac{\left(\sigma^{2}(s)\right)^2}{w(2\mid s)}}\right\}
    \sum_{d(s)\in[K]}e^{d(s)}(s) +o(1) \\
    &=
    \frac{1}{8}\frac{1}{\sqrt{MT}}
    \sum_{s\in\mathcal{S}}\inf_{\bm{w} \in \mathcal{W}}
    \left\{\sqrt{\frac{\left(\sigma^{1}(s)\right)^2}{w(1\mid s)} + \frac{\left(\sigma^{2}(s)\right)^2}{w(2\mid s)}}\right\} +o(1).
\end{align}

% According to the results of \citet{Laan2008TheCA}, \citet{Hahn2011}, and \citet{Kato2020adaptive}, the term
Consider the following optimization problem:
\[
\min_{w \in \mathcal{W}} \sqrt{\frac{\left(\sigma^{1}(s)\right)^2}{w(1\mid s)} + \frac{\left(\sigma^{2}(s)\right)^2}{w(2\mid s)}}.\]
The solution is
\[w(a\mid s) = \frac{\sigma^{a}(s)}{\sigma^{1}(s) + \sigma^{2}(s)},\]
and the optimal value is 
\[\sqrt{\Big(\sigma^{1}(s) + \sigma^{2}(s)\Big)^2}.\]

Therefore, we obtain the following lower bound:
\begin{align}
    &\frac{1}{8} \frac{1}{\sqrt{MT}} \sum_{s \in \mathcal{S}}
    \sqrt{\left(\sigma^1(s) + \sigma^2(s) \right)^2} +o(1)
    = \frac{1}{8}\frac{1}{\sqrt{T}} \mathop{\mathbb{E}}_{X \sim \xi} \sqrt{M \left(\sigma^1(X) + \sigma^2(X) \right)^2 } +o(1).
\end{align}
This completes the proof.
% \begin{align*}
%     \sup_{\stackrel{(e^1(s), e^2(s),\dots, e^K(s))\in[0, 1]^K}{\mathrm{s.t.}\ \sum_{a\in[K]}e^a(s)=1}}\inf_{\bm{w} \in \mathcal{W}}\sum_{a\in[K]}e^a(s) \frac{1}{8}\sqrt{M\Big(\sigma^{1}(s) + \sigma^{2}(s)\Big)^2}& = \frac{1}{8} \sqrt{M\Big(\sigma^{1}(s) + \sigma^{2}(s)\Big)^2}\sum_{a\in[K]}e^{a}(s).
% \end{align*}
% Since $\sum_{a\in[K]}e^{a}(s) = 1$, we obtain
% \begin{align}
%     \sqrt{T}R(P) \geq \mathop{\mathbb{E}}_{X \sim \zeta}\left[\frac{1}{8} \sqrt{M\Big(\sigma^{1}(X) + \sigma^{2}(X)\Big)^2}\right].
% \end{align}
% Thus, the proof is complete. 

\end{proof}

\section{Proof of Theorem~\ref{thm:regret_upper_bound}}
In the following sections, we prove the upper bound. To show the upper bound, we aim to use the result of \citet{ZhouZhengyuan2018OMPL}, which provides an upper bound for expected simple regret in the problem of policy learning with multiple treatment arms. Here, note that we cannot directly apply the results of \citet{ZhouZhengyuan2018OMPL} because they assume that the observations are i.i.d. in their study, but observations are non-i.i.d. in our study. 

\citet{ZhouZhengyuan2018OMPL} assumes that their outcomes are bounded. Therefore, to apply their result, for analysis, let us define the following quantities:
\begin{align*}
&\mu^a_c(P)(X) \coloneqq \mathbb{E}_P[c_T(Y^a)],\quad Q_c(P)(\pi) \coloneqq \mathop{\mathbb{E}}_{X \sim \zeta}\left[\sum_{a\in[K]}\pi(a\mid X)\mu^{a}_c(P)(X)\right],\\
&\pi^*_c(P) \coloneqq \argmax_{\pi\in\Pi} Q_c(P)(\pi),\quad R_c(P)(\pi) \coloneqq  Q_c(P)(\pi^*_c) - Q_c(P)(\pi). 
\end{align*}
These quantities are used in the conditions in the main theorem and the following proof. For $Q$ and $Q_c$, we state the following lemma. We omit the proof. 
\begin{lemma}
\label{asm:conv_pol}
There exist $0 < \alpha < 1/2$ and $U_T = T^\alpha$ such that and for any $\zeta$, and any $P\in\mathcal{P}_\zeta$, $\sup_{\pi\in\Pi}\Big|\sqrt{T}Q(P)(\pi) - \sqrt{T}Q_c(P)(\pi)\Big| \to 0$ holds as $T\to \infty$. 
%Additionally, for any $\zeta$, and any $P\in\mathcal{P}_\zeta$, $\max_{\pi \in \Pi}\sqrt{T}Q(P)(\pi) - \max_{\pi \in \Pi}\sqrt{T}Q_c(P)(\pi) \to 0$ holds as $T\to \infty$.
\end{lemma}
Because we set $U_T$ in $c_T(\cdot)$ as $U_T \to \infty$ when $T \to \infty$, this assumptions implies that when a clipping $c_T\left(Y_t\right) = \mathrm{thre}\Big(Y_t, U_T, -U_T\Big)$ asymptotically vanishes as $T\to \infty$, $\pi^*_c$ approaches $\pi^*$, an optimal policy under outcomes without the clipping. For example, this assumption is satisfied by assuming sub-Gaussianity about $Y(a)$. 

Thus, by clipping the outcomes using $c_T(\cdot)$ in $Q_c$, our problem satisfies the boundedness in \citet{ZhouZhengyuan2018OMPL}. To circumvent the issue of non-i.i.d. observations, we show an asymptotic equivalence between our empirical policy value $\widehat{Q}_T(\pi)$ and a hypothetical empirical policy value constructed from hypothetical i.i.d. observations in Section~\ref{sec:asymp_equiv}. Then, using the results of \citet{ZhouZhengyuan2018OMPL}, we upper bound a regret for $ Q_c(P)(\pi^{*}_c) -  Q_c(P)(\widehat{\pi})$ with the hypothetical policy value in Section~\ref{sec:upper_bound}. Finally, in Section~\ref{eq:asymp_opt}, we derive the upper bounds for $R(\pi^*) = Q(P)(\pi^{*}) -  Q(P)(\widehat{\pi})$ from the upper bounds for $ Q_c(P)(\pi^{*}_c) -  Q_c(P)(\widehat{\pi})$. 
\subsection{Asymptotic Equivalence}
\label{sec:asymp_equiv}
Let $w_t(0\mid x)$ be $0$. We write $\widehat{\Gamma}^a_t$ as
\begin{align}
    &\widehat{\Gamma}^a_t \coloneqq  \widehat{\Gamma}^a(Y^a_t, \xi_t, X_t)\\
    & 
 \coloneqq \frac{\mathbbm{1}\left[\sum^{a-1}_{b = 0}\widehat{w}_t(b\mid X_t) \leq \xi_t \leq \sum^{a}_{b = 0}\widehat{w}_t(a\mid X_t)\right]\big(c_T(Y^a_t) - \widehat{\mu}^a_t(X_t)\big)}{\widehat{w}_t(a\mid  X_t)} + \widehat{\mu}^a_t(X_t) - \mu^a_c(X_t).
\end{align}
This expression equals the original definition of $\widehat{\Gamma}^a_t$.
Let $w^*(0\mid x)$ be $0$. Then, we define
\begin{align}
    &\Gamma^{*a}_{t, c} \coloneqq \Gamma^{*a}_c(Y^a_t, \xi_t, X_t)\\
    &\coloneqq  \frac{\mathbbm{1}\left[\sum^{a-1}_{b = 0}{w}^*(b\mid X_t) \leq \xi_t \leq \sum^{a}_{b = 0}{w}^*(a\mid X_t)\right]\big(c_T(Y^a_t) - \mu^a_c(P)(X_t)\big)}{{w}^*(a\mid  X_t)}.
\end{align}
We also define 
\begin{align}
    &\Gamma^{*a}_{t} \coloneqq \Gamma^{*a}(Y^a_t, \xi_t, X_t)\\
    &\coloneqq  \frac{\mathbbm{1}\left[\sum^{a-1}_{b = 0}{w}^*(b\mid X_t) \leq \xi_t \leq \sum^{a}_{b = 0}{w}^*(a\mid X_t)\right]\big(Y^a_t - \mu^a(P)(X_t)\big)}{{w}^*(a\mid  X_t)}.
\end{align}

Here. we show that $\widehat{\Gamma}^a_t$ and $\Gamma^{*a}_t$ are asymptotically equivalent. The proof is shown in Appendix~\ref{appdx:lem:equiv}. 
\begin{lemma}
Suppose that Assumption~\ref{asm:consistent} holds. Then,
\label{lem:equiv}
    \begin{align}
        \sqrt{T} \widehat{Q}_T\left(\widehat{\pi}^{\mathrm{PLAS}}_T\right) = \sqrt{T}\sum^T_{t=1}\sum_{a\in[K]}\pi(a\mid X_t)\Gamma^{*a}_c(Y^a_t, \xi_t, X_t) + o_P(1). 
    \end{align}
    holds as $T\to \infty$. 
\end{lemma}

Denote $\sqrt{T}\sum^T_{t=1}\sum_{a\in[K]}\pi(a\mid  X_t)\Gamma^{*a}_t$ by $\widehat{Q}_T\left(\widehat{\pi}^{\mathrm{PLAS}}_T\right)$. Note that $\widehat{Q}_T\left(\widehat{\pi}^{\mathrm{PLAS}}_T\right)$ consists of only i.i.d. observations; therefore, we can directly apply the results of \citet{ZhouZhengyuan2018OMPL} to derive the upper bound for the policy regret.  

This technique of the asymptotic equivalence is inspired by \citet{Hahn2011} and is important because, for dependent observations (even if they are martingales), we cannot apply the tools for upper-bounding regrets or risks, such as the Rademacher complexity. For instance, \citet{zhan2022policy} addresses this issue and establishes an off-policy learning method from adaptively collected observations by utilizing the Rademacher complexity for martingales developed by \citet{Rakhlin2015}. However, our interest lies in developing upper bounds depending on variances because our lower bounds also depend on variances, and such upper bounds are considered to be tight. Although \citet{ZhouZhengyuan2018OMPL} derives such upper bounds for policy learning from i.i.d. observations using the local Rademacher complexity, it is unclear whether we can use the results of \citet{ZhouZhengyuan2018OMPL} with the Rademacher complexity for martingales developed by \citet{Rakhlin2015}. In contrast, in this study, if we restrict the problem to BAI and the evaluation metric to the worst-case expected simple regret, we show that we can apply the results of \citet{ZhouZhengyuan2018OMPL} and avoid the use of the Rademacher complexity for martingales by bypassing a hypothetical policy value that only depends on i.i.d. observations. 

\subsection{Upper Bound under I.I.D. Observations}
\label{sec:upper_bound}
Because $\widehat{Q}_T(\widehat{\pi}^{\mathrm{PLAS}}_T)$ is asymptotically equivalent to a policy value that consists of i.i.d. observations, we can apply the results of policy learning with i.i.d. observations to bound the policy regret. Specifically, we modify a regret upper bound shown by \citet{ZhouZhengyuan2018OMPL}, given as the following lemma.

\begin{lemma}[Modified upper bound]
\label{thm:iid_regret}
Suppose that Assumptions~\ref{asm:consistent} and \ref{assump:1} hold. Then, for any $\zeta$ and any $P \in\mathcal{P}_\zeta$, 
    \begin{align}
       &\mathbb{E}_P\Big[R_c(P)\left(\widehat{\pi}^{\mathrm{PLAS}}_T\right)\Big] = \mathbb{E}_P\Big[Q_c(P)\left(\pi^{*}_c\right) -  Q_c(P)\left(\widehat{\pi}^{\mathrm{PLAS}}_T\right)\Big]\\
       &= \mathbb{E}_P\left[\Big(54.4 \kappa(\Pi) + 435.2\Big)\frac{\Upsilon_{*}}{\sqrt{T}} + O\left(\frac{\sqrt{U_T}}{T^{3/4}}\right)\right]
    \end{align}
holds, where
\begin{align*}
    &\Upsilon_{*} =\mathbb{E}\left[\sqrt{\sup_{\pi_1,\pi_2 \in \Pi}\sum_{t\in[T]}\left\{ \sum_{a\in[K]}\Big(\pi_1(a\mid X_t)- \pi_2(a\mid X_t)\Big)\Gamma^{*a}(Y_t, \xi_t, X_t) \right\}^2}\right].
\end{align*}
\end{lemma}
The proof is shown in Appendix~\ref{appdx:proof_upperbound}. 

\subsection{Asymptotic Optimality}
\label{eq:asymp_opt}
Finally, we derive the upper bounds of $R(P)\left(\widehat{\pi}^{\mathrm{PLAS}}_T\right) = Q(P)(\pi^{*}) -  Q(P)(\widehat{\pi}^{\mathrm{PLAS}}_T)$ from the upper bounds of $R_c(P)\left(\widehat{\pi}^{\mathrm{PLAS}}_T\right) = Q_c(P)\left(\pi^{*}_c\right) -  Q_c(P)\left(\widehat{\pi}^{\mathrm{PLAS}}_T\right)$. 

By using Lemma~\ref{asm:conv_pol}, we can evaluate the regret $R(\pi^*)$ as follows:
\begin{align}
    &R(P)\left(\widehat{\pi}^{\mathrm{PLAS}}_T\right) = R(P)\left(\widehat{\pi}^{\mathrm{PLAS}}_T\right) - R_c(P)\left(\widehat{\pi}^{\mathrm{PLAS}}_T\right) + R_c(P)\left(\widehat{\pi}^{\mathrm{PLAS}}_T\right)\\
    &= Q(P)(\pi^{*}) -  Q(P)\left(\widehat{\pi}^{\mathrm{PLAS}}_T\right)\\
    &\ \ \ - \Big\{Q_c(P)\left(\pi^{*}_c\right) -  Q_c(P)\left(\widehat{\pi}^{\mathrm{PLAS}}\right)\Big\} + \Big\{Q_c(P)\left(\pi^{*}_c\right) -  Q_c(P)\left(\widehat{\pi}^{\mathrm{PLAS}}_T\right)\Big\}\\
    &= \Big\{Q(P)(\pi^{*}) - Q(P)(\pi^*_c)\Big\}\\
    &\ \ \ + \left(\Big\{Q(P)(\pi^*_c)- Q(P)\left(\widehat{\pi}^{\mathrm{PLAS}}_T\right)\Big\} - \Big\{Q_c(P)\left(\pi^{*}_c\right) -  Q_c(P)\left(\widehat{\pi}^{\mathrm{PLAS}}_T\right)\Big\}\right)\\
    &\ \ \ + \Big\{Q_c(P)\left(\pi^{*}_c\right) -  Q_c(P)\left(\widehat{\pi}^{\mathrm{PLAS}}_T\right)\Big\}\\
    &= \Big\{Q_c(P)\left(\pi^{*}_c\right) -  Q_c(P)\left(\widehat{\pi}^{\mathrm{PLAS}}_T\right)\Big\} + o(1/\sqrt{T}) = R_c(P)\left(\widehat{\pi}^{\mathrm{PLAS}}_T\right) + o(1/\sqrt{T}).
\end{align}

Let $U_T = T^{\alpha}$, where $\alpha$ is a value defined in Lemma~\ref{asm:conv_pol}. 

From Lemma~\ref{thm:iid_regret}, we have
\[\mathbb{E}_P\Big[R_c(P)\left(\widehat{\pi}^{\mathrm{PLAS}}_T\right)\Big]= \mathbb{E}_P\left[\Big(54.4 \kappa(\Pi) + 435.2\Big)\frac{\Upsilon_{*}}{\sqrt{T}} + O\left(\frac{\sqrt{U_T}}{T^{3/4}}\right)\right].\]

From the Cauchy-Schwarz inequality, we have
\begin{align*}
    \Upsilon_{*} &=\mathbb{E}\left[\sqrt{\sup_{\pi_1,\pi_2 \in \Pi}\sum_{t\in[T]}\left\{ \sum_{a\in[K]}\Big(\pi_1(a\mid X_t)- \pi_2(a\mid X_t)\Big)\Gamma^{*a}_c(Y_t, \xi_t, X_t) \right\}^2}\right]\\
    &\leq\mathbb{E}\left[\sqrt{\sup_{\pi_1,\pi_2 \in \Pi}\sum_{t\in[T]}\left\{ \sum_{a\in[K]}\Big(\pi_1(a\mid X_t)- \pi_2(a\mid X_t)\Big)\right\}^2\left\{ \sum_{a\in[K]}\Gamma^{*a}_c(Y_t, \xi_t, X_t) \right\}^2}\right]\\
    &\leq\mathbb{E}\left[\sqrt{\sup_{\pi_1,\pi_2 \in \Pi}\sum_{t\in[T]}4\left\{ \sum_{a\in[K]}\Gamma^{*a}_c(Y_t, \xi_t, X_t) \right\}^2}\right]\\
    &=2\mathbb{E}\left[\sqrt{\sum_{t\in[T]}\left\{ \sum_{a\in[K]}\Gamma^{*a}_c(Y_t, \xi_t, X_t) \right\}^2}\right].
    \end{align*}

From the law of iterated expectations and the Jensen inequality, we have
\begin{align*}
    \Upsilon_{*} &\leq 2\mathbb{E}\left[\sqrt{\sum_{t\in[T]}\left\{ \sum_{a\in[K]}\Gamma^{*a}_c(Y_t, \xi_t, X_t) \right\}^2}\right] \\
    &= 2\mathop{\mathbb{E}}_{X \sim \zeta}\left[\sqrt{\mathbb{E}\left[\sum_{t\in[T]}\left\{ \sum_{a\in[K]}\Gamma^{*a}_c(Y_t, \xi_t, X_t) \right\}^2\right]}\right]\\
    &=  2\mathop{\mathbb{E}}_{X \sim \zeta}\left[\sqrt{T\mathbb{E}\left[\left\{ \sum_{a\in[K]}\Gamma^{*a}_c(Y_t, \xi_t, X_t) \right\}^2\right]}\right].
\end{align*}

In conclusion, we obtain
\begin{align*}
    \Upsilon_{*} / \sqrt{T} &\leq 2\mathop{\mathbb{E}}_{X \sim \zeta}\left[\sqrt{\mathbb{E}\left[\left\{ \sum_{a\in[K]}\Gamma^{*a}_c(Y_t, \xi_t, X_t) \right\}^2\right]}\right]\\
    & =  2\mathop{\mathbb{E}}_{X \sim \zeta}\left[\sqrt{\mathbb{E}\left[\left\{ \sum_{a\in[K]}\Gamma^{*a}(Y_t, \xi_t, X_t) \right\}^2\right]}\right] + o(1)\\
    & = 2\mathop{\mathbb{E}}_{X \sim \zeta}\left[\sqrt{\sum_{a\in[K]}\frac{\left(\sigma^a(X)\right)^2}{w^*(a\mid  X)}}\right] + o(1),
\end{align*}
as $T\to \infty$ ($c_T \to \infty$). 

Similarly, we obtain
\begin{align*}
    \widetilde{\Upsilon}_{*} / \sqrt{T} \leq  2\sqrt{\sum_{a\in[K]}\mathop{\mathbb{E}}_{X \sim \zeta}\left[\frac{\left(\sigma^a(X)\right)^2}{w^*(a\mid  X)}\right]} + o(1).
\end{align*}
By substituting $w^*$ for each case with $K=2$ and $K\geq 3$, we obtain Theorem~\ref{thm:regret_upper_bound}.

\section{Proof of Lemma~\ref{lem:equiv}}
\label{appdx:lem:equiv}
In this section, we show 
    \begin{align}
    \label{eq:main}
        \sqrt{T} \widehat{Q}_T(\pi) = \sqrt{T}\frac{1}{T}\sum^T_{t=1}\sum_{a\in[K]}\pi(a\mid  X_t)\widehat{\Gamma}^a_t = \sqrt{T}\frac{1}{T}\sum^T_{t=1}\sum_{a\in[K]}\pi(a\mid  X_t)\Gamma^{*a}_t(Y^a_t, \xi_t, X_t) + o_P(1).
    \end{align}
Here, recall that
\begin{align*}
    &\frac{1}{\sqrt{T}}\sum^T_{t=1}\pi(a\mid  X_t)\widehat{\Gamma}^a_t(Y^a_t, \xi_t, X_t)\\
    &= \frac{1}{\sqrt{T}}\sum^T_{t=1}\pi(a\mid  X_t)\left\{\frac{\mathbbm{1}\left[\sum^{a-1}_{b = 0}\widehat{w}_t(b\mid X_t) \leq \xi_t \leq \sum^{a}_{b = 0}\widehat{w}_t(b\mid X_t)\right]\big(Y^a_t - \widehat{\mu}^a_t(X_t)\big)}{\widehat{w}_t(a\mid  X_t)} + \widehat{\mu}^a_t(X_t)\right\}.
\end{align*}
Therefore, to show \eqref{eq:main}, we show
\begin{align}
    &\frac{1}{\sqrt{T}}\sum^T_{t=1}\pi(a\mid  X_t)\left\{\frac{\mathbbm{1}\left[\sum^{a-1}_{b = 0}\widehat{w}_t(b\mid X_t) \leq \xi_t \leq \sum^{a}_{b = 0}\widehat{w}_t(b\mid X_t)\right]\big(Y^a_t - \widehat{\mu}^a_t(X_t)\big)}{\widehat{w}_t(a\mid  X_t)} + \widehat{\mu}^a_t(X_t)\right\}\\
    \label{eq:sub}
 &= \frac{1}{\sqrt{T}}\sum^T_{t=1}\pi(a\mid  X_t)\left\{\frac{\mathbbm{1}\left[\sum^{a-1}_{b = 0}{w}^*(b\mid X_t) \leq \xi_t \leq \sum^{a}_{b = 0}{w}^*(b\mid X_t)\right]\big(Y^a_t - \mu^a(P)(X_t)\big)}{{w}^*(a\mid  X_t)} + \mu^a(P)(X_t)\right\} + o_P(1).
\end{align}

\begin{proof}
Let us define
\begin{align}
    &G_t\big(Y^a_t, X_t, \xi_t; \left\{\widehat{w}_t(b\mid  X_t)\right\}_{b\in[K]}, \widehat{\mu}^a_T(X_t)\big)\\
    &\coloneqq  \pi(a\mid  X_t)\left\{\frac{\mathbbm{1}\left[\sum^{a-1}_{b = 0}\widehat{w}_t(b\mid X_t) \leq \xi_t \leq \sum^{a}_{b = 0}\widehat{w}_t(b\mid X_t)\right]\big(Y^a_t - \widehat{\mu}^a_t(X_t)\big)}{\widehat{w}_t(a\mid  X_t)} + \widehat{\mu}^a_t(X_t)\right\}.
\end{align}
Then, we obtain
\begin{align}
    &G_t\big(Y^a_t, X_t, \xi_t; \left\{\widehat{w}_t(b\mid  X_t)\right\}_{b\in[K]}, \widehat{\mu}^a_T(X_t)\big)\\
    &= G_t\big(Y^a_t, X_t, \xi_t; \left\{w^*(b\mid  X_t)\right\}_{b\in[K]}, \mu^a(P)(X_t)\big)\\
    &\ \ \ - G_t\big(Y^a_t, X_t, \xi_t; \left\{w^*(b\mid  X_t)\right\}_{b\in[K]}, \mu^a(P)(X_t)\big) + G_t\big(Y^a_t, X_t, \xi_t; \left\{\widehat{w}_t(b\mid  X_t)\right\}_{b\in[K]}, \widehat{\mu}^a_T(X_t)\big)\\
    &= G_t\big(Y^a_t, X_t, \xi_t; \left\{w^*(b\mid  X_t)\right\}_{b\in[K]}, \mu^a(P)(X_t)\big) + B_t,
\end{align}
where
\begin{align}
    &B_t \coloneqq  G_t\big(Y^a_t, X_t, \xi_t; \left\{\widehat{w}_t(b\mid  X_t)\right\}_{b\in[K]}, \widehat{\mu}^a_T(X_t)\big) - G_t\big(Y^a_t, X_t, \xi_t; \left\{w^*(b\mid  X_t)\right\}_{b\in[K]}, \mu^a(P)(X_t)\big).
\end{align}

To show \eqref{eq:sub}, we consider showing $\frac{1}{\sqrt{T}}\sum^T_{t=1}B_t \to 0$ as $T\to\infty$ in probability. 

We show $\frac{1}{\sqrt{T}}\sum^T_{t=1}B_t \to 0$ as $T\to\infty$ in probability by using the properties of martingales.
First, we have
\begin{align}
    &\mathbb{E}\left[B_t \mid X_t, \mathcal{F}_{t-1}\right]\\
    &= \mathbb{E}\left[\pi(a\mid  X_t)\left\{\frac{\mathbbm{1}\left[\sum^{a-1}_{b = 0}\widehat{w}_t(b\mid X_t) \leq \xi_t \leq \sum^{a}_{b = 0}\widehat{w}_t(b\mid X_t)\right]\big(Y^a_t - \widehat{\mu}^a_t(X_t)\big)}{\widehat{w}_t(a\mid  X_t)} + \widehat{\mu}^a_t(X_t)\right\}\mid X_t, \mathcal{F}_{t-1}\right]\\
    &\ \ \  - \mathbb{E}\left[\pi(a\mid  X_t)\left\{\frac{\mathbbm{1}\left[\sum^{a-1}_{b = 0}w^*(b\mid X_t) \leq \xi_t \leq \sum^{a}_{b = 0}w^*(b\mid X_t)\right]\big(Y^a_t - \mu^a(P)(X_t)\big)}{w^*(a\mid  X_t)} + \mu^a(P)(X_t)\right\}\mid X_t, \mathcal{F}_{t-1}\right]\\
    &= \mathbb{E}\left[\pi(a\mid  X_t)\left\{\frac{\widehat{w}_t(a\mid  X_t)\big(Y^a_t - \widehat{\mu}^a_t(X_t)\big)}{\widehat{w}_t(a\mid  X_t)} + \widehat{\mu}^a_t(X_t)\right\}\mid X_t, \mathcal{F}_{t-1}\right]\\
    &\ \ \ - \mathbb{E}\left[\pi(a\mid  X_t)\left\{\frac{w^*(a\mid  X_t)\big(Y^a_t - \mu^a(P)(X_t)\big)}{w^*(a\mid  X_t)} + \mu^a(P)(X_t)\right\}\mid X_t, \mathcal{F}_{t-1}\right]\\
    &= \pi(a\mid  X_t)\mu^a(P)(X_t) - \pi(a\mid  X_t)\mu^a(P)(X_t)\\
    &= 0.
\end{align}

This result implies that $\{B_t\}^T_{t=1}$ is a martingale difference sequence (MDS) because $\mathbb{E}[B_t\mid X_t, \mathcal{F}_{t-1}] = \mathbb{E}[\mathbb{E}[B_t\mid X_t, \mathcal{F}_{t-1}]] = 0$ holds. 
    
    Besides, from $\widetilde{w}_t(a\mid  X_t) - {w}^*(a\mid  X_t) \xrightarrow{\mathrm{a.s.}} 0$ and $\widehat{\mu}^a_t(X_t) -\mu^a(P)(X_t) \xrightarrow{\mathrm{a.s.}} 0$, $\mathbb{E}\left[B^2_t\mid X_t, \mathcal{F}_{t-1}\right]$ 
    also converges to zero almost surely as
    \begin{align}
        &\mathbb{E}[B^2_t\mid X_t, \mathcal{F}_{t-1}]\\
        &= \mathbb{E}\Bigg[\pi^2(a\mid  X_t)\Bigg(\left\{\frac{\mathbbm{1}\left[\sum^{a-1}_{b = 0}\widehat{w}_t(b\mid X_t) \leq \xi_t \leq \sum^{a}_{b = 0}\widehat{w}_t(b\mid X_t)\right]\big(Y^a_t - \widehat{\mu}^a_t(X_t)\big)}{\widehat{w}_t(a\mid  X_t)} + \widehat{\mu}^a_t(X_t)\right\}\\
    &\ \ \  - \left\{\frac{\mathbbm{1}\left[\sum^{a-1}_{b = 0}w^*(b\mid X_t) \leq \xi_t \leq \sum^{a}_{b = 0}w^*(b\mid X_t)\right]\big(Y^a_t - \mu^a(P)(X_t)\big)}{w^*(a\mid  X_t)} + \mu^a(P)(X_t)\right\}\Bigg)^2\mid X_t, \mathcal{F}_{t-1}\Bigg]\\
    &\xrightarrow{\mathrm{a.s.}} 0.
    \end{align}
    This is because from $\widehat{\mu}^a_t(X_t) -\mu^a(P)(X_t) \xrightarrow{\mathrm{a.s.}} 0$, 
    \begin{align}
        &\mathbb{E}\Bigg[\pi^2(a\mid  X_t)\Bigg(\left\{\frac{\mathbbm{1}\left[\sum^{a-1}_{b = 0}\widehat{w}_t(b\mid X_t) \leq \xi_t \leq \sum^{a}_{b = 0}\widehat{w}_t(b\mid X_t)\right]\big(Y^a_t - \widehat{\mu}^a_t(X_t)\big)}{\widehat{w}_t(a\mid  X_t)} + \widehat{\mu}^a_t(X_t)\right\}\\
        &\ \ \  - \left\{\frac{\mathbbm{1}\left[\sum^{a-1}_{b = 0}\widehat{w}(b\mid X_t) \leq \xi_t \leq \sum^{a}_{b = 0}\widehat{w}(b\mid X_t)\right]\big(Y^a_t - \mu^a(P)(X_t)\big)}{w^*(a\mid  X_t)} + \mu^a(P)(X_t)\right\}\Bigg)^2\mid X_t, \mathcal{F}_{t-1}\Bigg]\\
        &\xrightarrow{\mathrm{a.s.}} 0.
    \end{align}
    holds. Additionally, from $\widetilde{w}_t(a\mid  X_t) - {w}^*(a\mid  X_t) \xrightarrow{\mathrm{a.s.}} 0$, for any $\varepsilon > 0$, there exists $T(\varepsilon) > 0$ such that for any $t > T(\varepsilon)$,  $\left|\widetilde{w}_t(a\mid  x) - {w}^*(a\mid  x) \right| < \varepsilon$ holds for all $a\in[K]$ with probability one; that is, 
    \begin{align}
        &\mathbb{E}\Bigg[\pi^2(a\mid  X_t)\Bigg(\left\{\frac{\mathbbm{1}\left[\sum^{a-1}_{b = 0}\widehat{w}_t(b\mid X_t) \leq \xi_t \leq \sum^{a}_{b = 0}\widehat{w}_t(b\mid X_t)\right]\big(Y^a_t - \mu^a(P)(X_t)\big)}{\widehat{w}_t(a\mid  X_t)} + \mu^a(P)(X_t)\right\}\\
        &\ \ \  - \left\{\frac{\mathbbm{1}\left[\sum^{a-1}_{b = 0}w^*(b\mid X_t)  \leq \xi_t \leq \sum^{a}_{b = 0}w^*(b\mid X_t) \right]\big(Y^a_t - \mu^a(P)(X_t)\big)}{\widehat{w}_t(a\mid  X_t)} + \mu^a(P)(X_t)\right\}\Bigg)^2\mid X_t, \mathcal{F}_{t-1}\Bigg]\\
        &\leq \mathbb{E}\Bigg[\left\{\frac{\mathbbm{1}\left[\min\left\{\sum^{a-1}_{b = 0}\widehat{w}_t(b\mid X_t),\ \sum^{a-1}_{b = 0}w^*(b\mid X_t)\right\} \leq \xi_t \leq \max\left\{\sum^{a-1}_{b = 0}\widehat{w}_t(b\mid X_t),\ \sum^{a-1}_{b = 0}w^*(b\mid X_t)\right\} \right]}{\widehat{w}^2_t(a\mid  X_t)}\right\}\\
        &\ \ \ \ \ \ \ \ \ \ \ \ \ \ \ \ \ \ \ \ \ \ \ \ \ \ \ \ \ \ \ \ \ \ \ \ \ \ \ \ \ \ \ \ \ \ \ \ \ \ \ \ \ \ \ \ \ \ \ \ \ \ \ \ \ \ \ \ \ \ \ \ \ \ \ \ \ \ \ \ \ \ \ \ \ \ \ \ \ \ \ \ \ \ \ \ \times \pi^2(a\mid  X_t)\Big(Y^a_t - \mu^a(P)(X_t)\Big)^2\mid X_t, \mathcal{F}_{t-1}\Bigg]\\
         &+ \mathbb{E}\Bigg[\left\{\frac{\mathbbm{1}\left[\min\left\{\sum^{a}_{b = 0}\widehat{w}_t(b\mid X_t),\ \sum^{a}_{b = 0}w^*(b\mid X_t)\right\} \leq \xi_t \leq \max\left\{\sum^{a}_{b = 0}\widehat{w}_t(b\mid X_t),\ \sum^{a}_{b = 0}w^*(b\mid X_t)\right\} \right]}{\widehat{w}^2_t(a\mid  X_t)}\right\}\\
        &\ \ \ \ \ \ \ \ \ \ \ \ \ \ \ \ \ \ \ \ \ \ \ \ \ \ \ \ \ \ \ \ \ \ \ \ \ \ \ \ \ \ \ \ \ \ \ \ \ \ \ \ \ \ \ \ \ \ \ \ \ \ \ \ \ \ \ \ \ \ \ \ \ \ \ \ \ \ \ \ \ \ \ \ \ \ \ \ \ \ \ \ \ \ \ \ \times \pi^2(a\mid  X_t)\Big(Y^a_t - \mu^a(P)(X_t)\Big)^2\mid X_t, \mathcal{F}_{t-1}\Bigg]\\
        &\leq \mathbb{E}\left[\left\{\frac{\mathbbm{1}\left[\sum^{a-1}_{b = 0} \Big\{w^*(b\mid X_t) 
 - \varepsilon \Big\} \leq \xi_t \leq \sum^{a-1}_{b = 0} \Big\{w^*(b\mid X_t) 
 + \varepsilon \Big\} \right]}{\widehat{w}^2_t(a\mid  X_t)}\right\} \pi^2(a\mid  X_t)\Big(Y^a_t - \mu^a(P)(X_t)\Big)^2\mid X_t, \mathcal{F}_{t-1}\right]\\
         &+ \mathbb{E}\left[\left\{\frac{\mathbbm{1}\left[\sum^{a}_{b = 0} \Big\{w^*(b\mid X_t) 
 - \varepsilon \Big\} \leq \xi_t \leq \sum^{a}_{b = 0} \Big\{w^*(b\mid X_t) 
 + \varepsilon \Big\} \right]}{\widehat{w}^2_t(a\mid  X_t)}\right\} \pi^2(a\mid  X_t)\Big(Y^a_t - \mu^a(P)(X_t)\Big)^2\mid X_t, \mathcal{F}_{t-1}\right]\\
        &\leq \frac{2(a-1)\varepsilon}{\widehat{w}^2_t(a\mid  X_t)} \mathbb{E}\Bigg[\pi^2(a\mid  X_t)\Big(Y^a_t - \mu^a(P)(X_t)\Big)^2\mid \mathcal{F}_{t-1}\Bigg]+ \frac{2a\varepsilon}{\widehat{w}^2_t(a\mid  X_t)} \mathbb{E}\Bigg[\pi^2(a\mid  X_t)\Big(Y^a_t - \mu^a(P)(X_t)\Big)^2\mid X_t, \mathcal{F}_{t-1}\Bigg]\\
        &\leq C\varepsilon
    \end{align}
    holds with probability one, where $C > 0$ is a constant independent from $t$. This implies that for any $\varepsilon' > 0$, there exists $T'(\varepsilon') > 0$ such that for any $t > T'(\varepsilon')$, 
    \begin{align}
        &\Bigg|\mathbb{E}\Bigg[\pi^2(a\mid  X_t)\Bigg(\left\{\frac{\mathbbm{1}\left[\sum^{a-1}_{b = 0}\widehat{w}_t(b\mid X_t) \leq \xi_t \leq \sum^{a}_{b = 0}\widehat{w}_t(b\mid X_t)\right]\big(Y^a_t - \mu^a(P)(X_t)\big)}{\widehat{w}_t(a\mid  X_t)} + \mu^a(P)(X_t)\right\}\\
        &\ \ \ - \left\{\frac{\mathbbm{1}\left[\sum^{a-1}_{b = 0}w^*(b\mid X_t)  \leq \xi_t \leq \sum^{a}_{b = 0}w^*(b\mid X_t) \right]\big(Y^a_t - \mu^a(P)(X_t)\big)}{\widehat{w}_t(a\mid  X_t)} + \mu^a(P)(X_t)\right\}\Bigg)^2\mid X_t, \mathcal{F}_{t-1}\Bigg]\Bigg| < \varepsilon'
    \end{align}
    with probability one. We also have
    \begin{align}
    &\mathbb{E}\Bigg[\pi^2(a\mid  X_t)\Bigg(\left\{\frac{\mathbbm{1}\left[\sum^{a-1}_{b = 0}w^*(b\mid X_t) \leq \xi_t \leq \sum^{a}_{b = 0}w^*(b\mid X_t)\right]\big(Y^a_t - \mu^a(P)(X_t)\big)}{\widehat{w}_t(a\mid  X_t)} + \mu^a(P)(X_t)\right\}\\
        &\ \ \  - \left\{\frac{\mathbbm{1}\left[\sum^{a-1}_{b = 0}w^*(b\mid X_t)  \leq \xi_t \leq \sum^{a}_{b = 0}w^*(b\mid X_t) \right]\big(Y^a_t - \mu^a(P)(X_t)\big)}{w^*(a\mid X_t)} + \mu^a(P)(X_t)\right\}\Bigg)^2\mid X_t, \mathcal{F}_{t-1}\Bigg]\\
        &\xrightarrow{\mathrm{a.s.}} 0.
    \end{align}
    Thus, $\mathbb{E}[B^2_t\mid X_t, \mathcal{F}_{t-1}] \xrightarrow{\mathrm{a.s.}} 0$ holds. 
    
    Based on these results, from Chebyshev's inequality, $\frac{1}{\sqrt{T}}\sum^T_{t=1}B_t$ 
    converges to zero in probability. This is because for any $v > 0$, 
    \begin{align}
    \label{eq:chevyshev}
        &\mathbb{P}_P\left(\left|\frac{1}{\sqrt{T}}\sum^T_{t=1}B_t - \mathbb{E}\left[\frac{1}{\sqrt{T}}\sum^T_{t=1}B_t\right]\right| \geq v\right)= \mathbb{P}_P\left(\left|\frac{1}{\sqrt{T}}\sum^T_{t=1}B_t\right| \geq v\right) \leq \frac{\mathrm{Var}_P\left(\frac{1}{\sqrt{T}}\sum^T_{t=1}B_t\right)}{v^2}
    \end{align}

    Because $B_t$ is an MDS, the covariance between $B_t$ and $B_s$
    for $t\neq s$ is zero; that is, if $s < t$, $\mathrm{Cov}(B_t, B_s) = \mathbb{E}\left[B_tB_s\right] = \mathbb{E}\left[B_s\mathbb{E}\left[B_t\mid \mathcal{F}_{t-1}\right]\right] = 0$. 
    
    Therefore, we can show $\frac{1}{\sqrt{T}}\sum^T_{t=1}B_t \xrightarrow{\mathrm{p}} 0$ by showing 
    \begin{align}
    \label{eq:conv_var}
        \mathrm{Var}_P\left(\frac{1}{\sqrt{T}}\sum^T_{t=1}B_t\right) = \frac{1}{T}\sum^T_{t=1}\mathrm{Var}\left(B_t\right) \to 0,
    \end{align}
    as $T \to \infty$, where we used that the covariance between $B_t$ and $B_s$ for $t\neq s$ is zero. 
    
    To show \eqref{eq:conv_var}, we show
    \begin{align}
        \frac{1}{T}\sum^T_{t=1}\mathrm{Var}_P\left(B_t\mid X_t,  \mathcal{F}_{t-1}\right) =  \frac{1}{T}\sum^T_{t=1}\mathbb{E}_P[B^2_t\mid X_t, \mathcal{F}_{t-1}] \xrightarrow{\mathrm{p}} 0,
    \end{align}
    by using $\mathbb{E}[B^2_t\mid X_t, \mathcal{F}_{t-1}]\xrightarrow{\mathrm{a.s.}} 0$.

Let $u_t$ be $u_t =  \mathbb{E}_{P}\left[B^2_t\mid X_t, \mathcal{F}_{t-1}\right]$. Fix some positive $\epsilon > 0$ and $\delta > 0$. Almost-sure convergence of $u_t$ to zero as $t\to\infty$ implies that we can find a large enough $t(\epsilon)$ such that $|u_t| < \epsilon$ for all $t\geq t(\epsilon)$ with probability at least $1 - \delta$. Let $\mathcal{E}(\epsilon)$ denote the event in which this happens; that is, $\mathcal{E}(\epsilon) = \{ |u_t| < \epsilon\quad \forall\ t \geq t(\epsilon)\}$. Under this event, for $T > t(\epsilon)$, 
\begin{align*}
    \sum^T_{t=1}|u_t| \leq \sum^{t(\epsilon)}_{t=1} C + \sum^{T}_{t=t(\epsilon)n + 1} \epsilon = t(\epsilon) C + T\epsilon.  
\end{align*}
Therefore, we obtain
\begin{align*}
    \mathbb{P}\left( \frac{1}{T}\sum^T_{t=1}|u_t| > 2\epsilon \right)&= \mathbb{P}\left( \left\{\frac{1}{T}\sum^T_{t=1}|u_t| > 2\epsilon \right\}\cap \mathcal{E}(\epsilon)\right) + \mathbb{P}\left( \left\{\frac{1}{T}\sum^T_{t=1}|u_t| > 2\epsilon \right\}\cap \mathcal{E}^c(\epsilon)\right)\\
    &\leq \mathbb{P}\left( \frac{t(\epsilon)}{T} C + \epsilon > 2\epsilon \right) + \mathbb{P}\left(\mathcal{E}^c(\epsilon)\right)= \mathbb{P}\left( \frac{t(\epsilon)}{T} C > \epsilon \right) + \mathbb{P}\left(\mathcal{E}^c(\epsilon)\right).
\end{align*}
Letting $T\to \infty$, for arbitrarily small $\delta > 0$, , $\frac{1}{T}\sum^T_{t=1}\mathbb{E}_P[B^2_t\mid X_t, \mathcal{F}_{t-1}] \xrightarrow{\mathrm{p}} 0$ as $T\to\infty$ holds. 

Then, from the dominated convergence theorem and the boundedness of $B^2_t$, $\frac{1}{T}\sum^T_{t=1}\mathrm{Var}_P\left(B_t\right) \xrightarrow{\mathrm{p}} 0$ as $T\to\infty$ holds\footnote{Our proof for this part refers to the proof of Lemma~10 in \citet{hadad2019}.}. 

Therefore, from \eqref{eq:chevyshev}, $\frac{1}{\sqrt{T}}\sum^T_{t=1}B_t \xrightarrow{\mathrm{p}} 0$ holds, which implies
\begin{align}
    &\frac{1}{\sqrt{T}}\sum^T_{t=1}\pi(a\mid  X_t)\widehat{\Gamma}^a_t(Y^a_t, \xi_t, X_t)\\
 &= \frac{1}{\sqrt{T}}\sum^T_{t=1}\pi(a\mid  X_t)\left\{\frac{\mathbbm{1}\left[\sum^{a-1}_{b = 0}{w}^*(b\mid X_t) \leq \xi_t \leq \sum^{a}_{b = 0}{w}^*(a\mid X_t)\right]\big(Y^a_t - \mu^a(P)(X_t)\big)}{{w}^*(a\mid  X_t)} + \mu^a(P)(X_t)\right\}\\
 &\ \ \ \ \ \ + o_P(1)
\end{align}
as $T\to\infty$. 
\end{proof}

\section{Proof of Lemma~\ref{thm:iid_regret}}
\label{appdx:proof_upperbound}
This section provides the proof of Lemma~\ref{thm:iid_regret}. In Appendix~\ref{sec:prelim}, we introduce an upper bound of the Rademacher complexity shown by \citet{ZhouZhengyuan2018OMPL}. Then, in Appendix~\ref{appdx:proof_iid}, we prove a new lemma (Lemma~\ref{lem:iid_regret2}), which directly yields Lemma~\ref{thm:iid_regret}.

\subsection{Upper bound of the Rademacher complexity.}
\label{sec:prelim}
Let us define a policy class
\[\Pi^D \coloneqq \left\{h:\mathcal{X}\times \prod^K_{a=1}\mathbb{R}\to \mathbb{R} \mid h\big(x, (\Gamma^a)_{a\in[K]}\big) \coloneqq \sum_{a\in[K]}\Big( \pi_1(a\mid x) - \pi_2(a\mid x)\Big)\Gamma^a, \pi_1,\pi_2 \in \Pi\right\}\]
Then, let us define the Rademacher complexity as follows.
\begin{definition}
    Let $\{Z_t\}^T_{t=1}$ be a sequence of i.i.d. Rademacher random variables $Z_t \in \{-1, +1\}$: $\mathbb{P}[Z_t = +1] = \mathbb{P}[Z_t = -1] = \frac{1}{2}$.
    \begin{itemize}
        \item The empirical Rademacher complexity $\mathfrak{R}_T\left(\Pi^D; \big\{X_t, \Gamma^{*a}_t\big\}^T_{t=1}\right)$ of a function class $\Pi^D$
        is defined as
        \begin{align}
        &\mathfrak{R}_T\left(\Pi^D; \big\{X_t, \Gamma^{*a}_t\big\}^T_{t=1}\right)\\
        &= \mathbb{E}\left[\sup_{\pi_1, \pi_2 \in \Pi}\frac{1}{T}\left|\sum^T_{t=1} Z_t \sum_{a\in[K]} \Big(\pi_1(a\mid X_t) - \pi_2(a\mid X_t) \Big)\Gamma^{*a}_t \right| \mid \Big\{X_t, \Gamma^{*a}_t\Big\}^T_{t=1}\right],
    \end{align}
    where the expectation is taken with respect to $Z_1,\dots, Z_t$.
    \item The Rademacher complexity $\mathfrak{R}_T\left(\Pi^D\right)$ of the function class $\Pi^D$ is the expected value taken with respect to the observations $\Big\{X_t, \Gamma^{*a}_t\Big\}^T_{t=1}$ of the empirical Rademacher complexity: 
    \[\mathfrak{R}_T\left(\Pi^D\right) \coloneqq \mathbb{E}\left[\mathfrak{R}_T\left(\Pi^D; \big\{X_t, \Gamma^{*a}_t\big\}^T_{t=1}\right)\right].\] 
 \end{itemize}
\end{definition}

Our proof starts from the following result about the Rademacher complexity shown by \citet{ZhouZhengyuan2018OMPL}.
\begin{lemma}[From the inequality above (C.14) and the inequality in (C.19) of \citet{ZhouZhengyuan2018OMPL}]
\label{lem:iid_regret}
Suppose that Assumptions~\ref{asm:consistent} and \ref{assump:1} hold. Then, for any $\zeta$ and $P \in\mathcal{P}_\zeta$, 
    \begin{align}
       \mathfrak{R}_T\left(\Pi^D\right) \leq  13.6\sqrt{2}\left\{\kappa(\Pi)+8\right\}\frac{\Upsilon_{*}}{\sqrt{T}} +  + O\left(\frac{\sqrt{U_T}}{T^{3/4}}\right).
    \end{align}
holds, where
    \begin{align*}
    \Upsilon_{*} &=\mathbb{E}\left[\sqrt{\sup_{\pi_1,\pi_2 \in \Pi}\sum_{t\in[T]}\left\{ \sum_{a\in[K]}\Big(\pi_1(a\mid Z_i)- \pi_2(a\mid Z_i)\Big)\Gamma^{*a}(Y_t, \xi_t, X_t) \right\}^2}\right]. 
\end{align*}
\end{lemma}

Following this upper bound, \citet{ZhouZhengyuan2018OMPL} applies the Jensen inequality to bound $\Upsilon_{*}$ by 
\begin{align*}
    \Upsilon_{*} \leq \sqrt{\mathbb{E}\left[\sup_{\pi_1,\pi_2 \in \Pi}\left\{ \sum_{a\in[K]}\Big(\pi_1(a\mid Z_i)- \pi_2(a\mid Z_i)\Big)\Gamma^{*a}(Y_t, \xi_t, X_t) \right\}^2\right]}.
\end{align*}
Then, they apply the Talagrand inequality to bound the regret as
\[\Upsilon_{*} \leq \sqrt{\sup_{\pi_1,\pi_2 \in \Pi}\mathbb{E}\left[\left\{ \sum_{a\in[K]}\Big(\pi_1(a\mid Z_i)- \pi_2(a\mid Z_i)\Big)\Gamma^{*a}(Y_t, \xi_t, X_t) \right\}^2\right]}.\]

However, the use of the Jensen inequality yields a loose upper bound, which results in a mismatch between the upper bound and our derived lower bound. Therefore, in our proof, we consider bounding $\Upsilon_{*}$ without using Jensen's inequality.

\subsection{Proof of Lemma~\ref{thm:iid_regret}}
\label{appdx:proof_iid}
Based on this proof strategy, we show the following lemma.
\begin{lemma}[From the equation above (C.14) in \citet{ZhouZhengyuan2018OMPL}]
\label{lem:iid_regret2}
Suppose that Assumptions~\ref{asm:consistent} and \ref{assump:1} hold. Then, for each $P \in\mathcal{P}$, 
\begin{align}
    \mathbb{E}_P\left[R_c(P)\left(\widehat{\pi}^{\mathrm{PLAS}}_T\right)\right] \leq  54.4\sqrt{2}\left\{\kappa(\Pi)+8\right\}\frac{\Upsilon_{*}}{\sqrt{T}} +  + O\left(\frac{\sqrt{U_T}}{T^{3/4}}\right).
\end{align}
holds, where $\Upsilon_{*}$ is defined in Lemma~\ref{lem:iid_regret}.
\end{lemma}
Lemma~\ref{thm:iid_regret} directly follows from Lemma~\ref{lem:iid_regret2} by multiplying both sides by $\sqrt{T}$ and letting $T \to \infty$. 

\begin{proof}[Proof of Lemma~\ref{lem:iid_regret2}]
Recall that
\begin{align}
    \widehat{Q}_T(\pi) \coloneqq  \frac{1}{T}\sum^T_{t=1}\sum_{a\in[K]}\pi(a\mid  X_t)\widehat{\Gamma}^a_t,
\end{align}
Let us define the following quantities:
\begin{align}
    \widetilde{Q}_T\left(\pi\right) &\coloneqq  \frac{1}{T}\sum^T_{t=1}\sum_{a\in[K]}\pi(a\mid  X_t)\Gamma^{*a}_c(Y^a_t, \xi_t, X_t),\\
    \widehat{\Delta}\big(\pi_1, \pi_2\big) &\coloneqq \widehat{Q}_T\left(\pi_1\right) - \widehat{Q}_T\left(\pi_2\right),\\
    \widetilde{\Delta}\big(\pi_1, \pi_2\big) &\coloneqq \widetilde{Q}_T\left(\pi_1\right) - \widetilde{Q}_T\left(\pi_2\right),\\
    \Delta\big(\pi_1, \pi_2\big) &\coloneqq Q_c(P)\left(\pi_1\right) - Q_c(P)\left(\pi_2\right).
\end{align}

Then, we have
\begin{align}
    &R_c(P)\left(\widehat{\pi}^{\mathrm{PLAS}}_T\right)\\
    &= Q_c(P)\left(\pi^{*}_c\right) -  Q_c(P)\left(\widehat{\pi}^{\mathrm{PLAS}}\right)\\
    &= \widehat{Q}_T\left(\pi^{*}_c\right) - \widehat{Q}_T\left(\pi^{\mathrm{PLAS}}\right) + \Delta\left(\pi^{*}_c, \pi^{\mathrm{PLAS}}\right) - \widehat{\Delta}\left(\pi^{*}_c, \pi^{\mathrm{PLAS}}\right)\\
    &\leq \left|\Delta\left(\pi^{*}_c, \pi^{\mathrm{PLAS}}\right) - \widehat{\Delta}\left(\pi^{*}_c, \pi^{\mathrm{PLAS}}\right)\right|\\
    &\leq \sup_{\pi_1, \pi_2 \in \Pi}\left|\Delta\big(\pi_1, \pi_2\big) - \widehat{\Delta}\big(\pi_1, \pi_2\big)\right|\\
    &\leq \sup_{\pi_1, \pi_2 \in \Pi}\left|\widehat{\Delta}\big(\pi_1, \pi_2\big) - \widehat{\Delta}\big(\pi_1, \pi_2\big)\right| + \sup_{\pi_1, \pi_2 \in \Pi}\left|\Delta\big(\pi_1, \pi_2\big) - \widetilde{\Delta}\big(\pi_1, \pi_2\big)\right|.
\end{align}
From Lemma~\ref{lem:equiv}, we have
\begin{align}
    &\sqrt{T}R_c(P)\leq \sup_{\pi_1, \pi_2 \in \Pi}\left|\Delta\big(\pi_1, \pi_2\big) - \widetilde{\Delta}\big(\pi_1, \pi_2\big)\right|+ o_P(1).
\end{align}

Therefore, we obtain 
\begin{align}
    \mathbb{E}_P\left[R_c(P)\left(\widehat{\pi}^{\mathrm{PLAS}}_T\right)\right] \leq \mathbb{E}_P\left[\sup_{\pi_1, \pi_2 \in \Pi}\left|\Delta\big(\pi_1, \pi_2\big) - \widetilde{\Delta}\big(\pi_1, \pi_2\big)\right|\right],
\end{align}
where we used that $\sup_{\pi_1, \pi_2 \in \Pi}\left|\widehat{\Delta}\big(\pi_1, \pi_2\big) - \widehat{\Delta}\big(\pi_1, \pi_2\big)\right|$ is a bounded random variable.

Then, from the property of the Rademacher complexity \citep{BartlettMendelson2003} and Lemma~\ref{lem:iid_regret}, we have
     \begin{align}
       &\mathbb{E}_P\left[\sup_{\pi_1, \pi_2 \in \Pi}\left|\Delta\big(\pi_1, \pi_2\big) - \widetilde{\Delta}\big(\pi_1, \pi_2\big)\right|\right]\\
       &\leq 4\mathfrak{R}_T(\Pi)\\
       &\leq 54.4\sqrt{2}\left\{\kappa(\Pi)+8\right\}\frac{\Upsilon_{*}}{\sqrt{T}} +  + O\left(\frac{\sqrt{U_T}}{T^{3/4}}\right).
    \end{align}
This completes the proof.
\end{proof}

\end{document}